\documentclass{article}
\newcommand{\normw}[2]{\|#1\|_{#2}}
\newcommand{\realset}{\mathbb{R}}
\newcommand{\T}{^\top}
\usepackage{microtype}
\usepackage{graphicx}
\usepackage{subfigure}
\usepackage{booktabs} 

\usepackage{algorithm}
\usepackage[noend]{algpseudocode}
\usepackage{amsfonts}
\usepackage{amsmath}
\usepackage{amssymb}
\usepackage{amsthm}
\usepackage{bbm}
\usepackage{bm}
\usepackage{color}
\usepackage{dirtytalk}
\usepackage{dsfont}
\usepackage{enumerate}
\usepackage{listings}
\usepackage{mathtools}
\usepackage{natbib}
\usepackage{subfigure}
\usepackage{xspace}
\usepackage{float}     
\usepackage{amsmath,amsthm}
\usepackage{thmtools, thm-restate} 
\usepackage[dvipsnames]{xcolor}
\usepackage[bookmarks=false]{hyperref}
\hypersetup{
  pdffitwindow=true,
  pdfstartview={FitH},
  pdfnewwindow=true,
  colorlinks,
  linktocpage=true,
  urlcolor=Green,
  linkcolor=Green,
  citecolor=Green
}
\usepackage[capitalize,noabbrev]{cleveref}

\usepackage[]{color-edits}
\addauthor{as}{blue}
\addauthor{rd}{purple}

\usepackage{hyperref}

\usepackage[accepted]{icml2025}

\usepackage{amsmath}
\usepackage{amssymb}
\usepackage{mathtools}
\usepackage{amsthm}
\usepackage{xspace}

\usepackage[capitalize,noabbrev]{cleveref}

\theoremstyle{plain}
\newtheorem{theorem}{Theorem}[section]
\newtheorem{proposition}[theorem]{Proposition}
\newtheorem{lemma}[theorem]{Lemma}

\theoremstyle{definition}

\newtheorem{assumption}[theorem]{Assumption}
\theoremstyle{remark}

\usepackage[textsize=tiny]{todonotes}

\icmltitlerunning{FisherSFT: Data-Efficient Supervised Fine-Tuning of Language Models Using Information Gain}
\usepackage[dvipsnames]{xcolor}
\newcommand{\beq}{\begin{equation}}
\newcommand{\eeq}{\end{equation}}

\renewcommand\vec[1]{\operatorname{vec}#1}

\newcommand\E{\mathbb{E}}

\newcommand\I{\mathbb{I}}

\newcommand\R{\mathbb{R}}

\newcommand\1{\mathbbm{1}}

\newcommand{\x}{\mathbf{x}}

\renewcommand{\I}{\mathbb{I}}

\newcommand{\cL}{{\cal L}}

\newcommand{\cN}{{\cal N}}

\newcommand{\cS}{{\cal S}}

\newcommand{\cB}{{\cal B}}
\newcommand{\cD}{{\cal D}}

\newcommand{\bI}{\mathbf{I}}

\DeclareMathOperator*{\argmax}{argmax}
\DeclareMathOperator*{\argmin}{argmin}

\usepackage{thm-restate}

\theoremstyle{definition}

\newcommand {\commentout}[1] {}

\def\ints{{{\rm Z} \kern -.35em {\rm Z} }}  
\def\smallints{{{\rm Z} \kern -.3em {\rm Z} }}  
\def\pints{{{\rm I} \kern -.15em {\rm N} }}      
\newcommand{\reals}{\mathbb R}

\def\cplx{{{\rm I} \kern -.45em {\rm C} }}       
\def\l2{\rm {\mathcal L}^{2}(\reals)}            

\newcommand{\be}{\begin{eqnarray}}
\newcommand{\ee}{\end{eqnarray}}
\newcommand{\bea}{\begin{eqnarray}}
\newcommand{\eea}{\end{eqnarray}}
\newcommand{\beaa}{\begin{eqnarray*}}
\newcommand{\eeaa}{\end{eqnarray*}}
\newcommand{\bnad}{\begin{nad}}
\newcommand{\enad}{\end{nad}}

\newcommand{\IGNORE}[1]{}

\newcommand{\set}[1] {\{ #1\}}

\let\strokeL\L
\DeclareRobustCommand{\L}{\ifmmode\mathbf{L}\else\strokeL\fi}

\def\ones{{\mathbf{1}}}

\let\trace\relax
\DeclareMathOperator{\trace}{tr}

\newcommand{\sentenceod}{\ensuremath{\color{Green}\tt SentenceOD}\xspace}
\newcommand{\tokenod}{\ensuremath{\color{Green}\tt TokenOD}\xspace}
\newcommand{\uniform}{\ensuremath{\color{Green}\tt Uniform}\xspace}

\newcommand{\density}{\ensuremath{\color{Green}\tt DensitySampling}\xspace}

\newcommand{\cluster}{\ensuremath{\color{Green}\tt ClusteredSampling}\xspace}

\newcommand{\askLLM}{\ensuremath{\color{Green}\tt AskLLM}\xspace}

\newcommand{\FisherSFT}{\ensuremath{\color{Green}\tt FisherSFT}\xspace}
\begin{document}

\twocolumn[
\icmltitle{FisherSFT: Data-Efficient Supervised Fine-Tuning of Language Models Using Information Gain}



\icmlsetsymbol{equal}{*}
\begin{icmlauthorlist}
\icmlauthor{Rohan Deb}{uiuc}
\icmlauthor{Kiran Thekumparampil}{amazon}
\icmlauthor{Kousha Kalantari}{amazon}
\icmlauthor{Gaurush Hiranandani}{typeface}
\icmlauthor{Shoham Sabach}{amazon,technion}
\icmlauthor{Branislav Kveton}{adobe}
\end{icmlauthorlist}

\icmlaffiliation{uiuc}{University of Illinois, Urbana-Champaign \emph{(Work done while interning at Amazon)}}
\icmlaffiliation{amazon}{Amazon}
\icmlaffiliation{typeface}{Typeface}
\icmlaffiliation{technion}{Technion}
\icmlaffiliation{adobe}{Adobe Research}

\icmlcorrespondingauthor{Rohan Deb}{rd22@illinois.edu}

\icmlkeywords{Active Learning, Data Collection}

\vskip 0.3in
]



\printAffiliationsAndNotice{}

\begin{abstract}
Supervised fine-tuning (SFT) is a standard approach to adapting large language models (LLMs) to new domains. In this work, we improve the statistical efficiency of SFT by selecting an informative subset of training examples. Specifically, for a fixed budget of training examples, which determines the computational cost of fine-tuning, we determine the most informative ones. The key idea in our method is to select examples that maximize information gain, measured by the Hessian of the log-likelihood of the LLM. We approximate it efficiently by linearizing the LLM at the last layer using multinomial logistic regression models. Our approach is computationally efficient, analyzable, and performs well empirically. We demonstrate this on several problems, and back our claims with both quantitative results and an LLM evaluation.
\end{abstract}

\section{Introduction}
\label{sec:intro}
\emph{Large language models (LLMs)} \citep{bommasani21opportunities} have emerged as general purpose tools that can match human performance in both zero-shot and few-shot settings \citep{radford2019language,brown20language}. LLMs are typically trained in three stages \citep{ouyang22training}: pre-training on a large corpus of diverse text, supervised fine-tuning in the domain of interest \citep{wei22finetuned}, and alignment to human preferences \citep{ouyang22training,rafailov23direct}. The main challenge in all stages is the sheer scale of LLMs, which increased by four orders of magnitude in just four years: from $117$ million parameters in GPT-2 (2019) to $1.76$ trillion parameters in GPT-4 (2023).

We focus on \emph{supervised fine-tuning (SFT)} \citep{wei22finetuned} in this work. A standard approach in SFT is to optimize a \emph{low-rank adapter (LoRA)} \citep{hu22lora}. The key idea in LoRA is to add low-rank matrices to the matrices in the transformer layers. During fine-tuning, only the low-rank matrices are adapted. Therefore, the computational cost of LoRA is linear in the rank of the low-rank matrices, which naturally trades off the computational cost for the quality of the approximation. The simplicity of LoRA made it popular in practice and thousands of different adapters have been trained \citep{peft}. We propose a complementary approach that selects a subset of most informative training examples for fine-tuning. The computational cost of fine-tuning is linear in the size of the chosen subset. Therefore, as in LoRA, the number of chosen examples naturally trades off the computational cost of fine-tuning for quality.

The idea of selecting better training examples for SFT is not new and has been explored extensively before. Coverage‐based approaches select sufficiently diverse examples to form coresets \cite{phillips2017coresets,tukan2021coresets}. Quality‐based sampling prioritizes weeding out low‐value or unhelpful examples \cite{wenzek2019ccnet,muenchigoff2023scaling}. In ASK-LLM \cite{sachdeva2024how}, a proxy LLM is prompted with a potential training example and asked whether the example should be used for training. We review all of these approaches in detail in \cref{sec:related}. The main difference in our work is that we choose training examples based the log-likelihood of the LLM and thus take their information value into account. 

Without loss of generality, we view training examples in fine-tuning as sentences, each being a sequence of tokens. We want to select the most informative $n$ sentences, which determines the computational cost of fine-tuning. We do this based on the SFT objective. Specifically, note that the last layer of the LLM is a product of next token probabilities. Each probability is represented by a multinomial logistic regression model \citep{bishop06pattern}, where the feature vector is the embedding of all previous tokens. 
Therefore, the problem of selecting the most informative sentences for fine-tuning can be viewed as a variant of an optimal design \citep{Pukelsheim1993OptimalDO,stufken12optimal} for multinomial logistic regression, where tokens in a sentence are chosen jointly based on their information value. We derive an efficient approximation to the Hessian of the LLM log-likelihood, which represents how informative a set of sentences is, and then optimize a lower bound on its log determinant to find the most informative sentences.

We make the following contributions.

\textbf{(1)} We establish a connection between the supervised fine-tuning objective of LLMs and a product of multinomial logistic regression models in \cref{sec:prelim}. 

\textbf{(2)} 
We propose our method in \cref{sec:algo}. Our main technical contribution is a computationally-efficient approximation to the log determinant of the Hessian of the log-likelihood. More specifically, all matrices in this optimization problem are $d \times d$, where $d$ is the size of transformer embeddings, as opposing to $d L \times d L$, where $L$ is the number of distinct tokens. We solve the optimization problem greedily \citep{nemhauser78approximation}, using the monotonicity and submodularity of the objective. At a high level, our algorithm greedily chooses sentences with tokens that are jointly most informative. This is in a stark contrast to treating each sentence as a single data point \citep{das24active,mukherjee24optimal,thekumparampil24comparing,liu24dual,scheid24optimal}, which we compare to in \cref{sec:experiments}.

\textbf{(3)} We analyze our method in \cref{sec:proof}. The main result of our analysis is that the prediction error of our model decreases at rate $\tilde{O}(d L / \sqrt{n})$, where $n$ is the number of chosen sentences. The dependence on $n$ is similar to other recent results in the literature \citep{zhu23principled,mukherjee24optimal,thekumparampil24comparing}.

\textbf{(4)} We evaluate our method empirically in \cref{sec:experiments}. Our experiments are both synthetic and on real-world data with GPT models. We observe that our approach leads to lower prediction errors than the baselines. We also conduct a qualitative evaluation of learned GPT models by a larger LLM.

\section{Problem Formulation}
\label{sec:prelim}
We have a dataset of $N$ sentences indexed by $i \in [N]$. 
A sentence $i$ consists of $M_i$ tokens indexed by $j \in [M_i]$. 
Let $y_{i,j}$ be the $j$-th token in sentence $i$. 
Each token $y_{i,j} \in [L]$ belongs to a vocabulary of size $L$. 
We represent sentence $i$ by the sequence of its tokens,
\[
   y_i \;=\; (y_{i,1},\,y_{i,2},\,\dots,\,y_{i,M_i}),
\]
and denote the entire dataset by $\mathcal{D} = \{\,y_i\}_{i \in [N]}$. To model the evolution of each sentence token-by-token, we define a vector $\mathbf{x}_{i,j} \in \mathbb{R}^d$ that captures the relevant \emph{history} up to the $j$-th token in sentence $i$. In the simplest setting, $\mathbf{x}_{i,j}$ may be a word embedding of $y_{i, j - 1}$. In a large language model, $\mathbf{x}_{i,j}$ could be the output of the pre-logit layer that encodes contextual information about tokens $y_{i,1}, y_{i,2}, \dots, y_{i,j-1}$.

\textbf{Objective: } Our objective is to select an $n$ sized subset $\cS \subset [N]$ of sentences from the dataset $\cD$ and subsequently fine-tune a model using this selected set $\cS$. For fine-tuning an LLM we use pre-logit layer embeddings of sentences to compute this subset $\cS$. 

We denote the parameter matrix by $\Theta_* \in \mathbb{R}^{d \times L}$. Its $\ell$-th column $\theta^{*}_{\ell} \in \mathbb{R}^d$ corresponds 
to the last-layer LLM parameters for token $\ell \in [L]$, i.e., $\Theta_* = (\theta^{*}_{\ell})_{\ell \in [L]}$. 
Under a softmax model, the probability of observing token $\ell$ at 
position $(i,j)$ is given by
\begin{align}
\label{eq:softmax}
   p\bigl(\ell \,\big\vert\, \mathbf{x}_{i,j};\, \Theta_* \bigr)
   \;=\;
   \frac{\exp\!\Bigl(\theta^{*\top}_{\ell}\,\mathbf{x}_{i,j}\Bigr)}
        {\sum\limits_{k=1}^{L} \exp\!\Bigl(\theta^{*\top}_{k}\,\mathbf{x}_{i,j}\Bigr)},
\end{align} 
Under such a softmax model, the goal of an autoregressive model is to learn an estimate of the unknown parameter matrix $\Theta_*$ by minimizing the negative log-likelihood
\begin{align}
\label{eq:loglik}
    \cL(\Theta) = -\frac{1}{N}\sum_{i \in [N]} \sum_{j = 1}^{M_i} \log p(y_{i,j}|\x_{i,j};\Theta).
\end{align}
Our objective is to select an $n$ sized subset $\cS \subset [N]$ of sentences from the dataset $\cD$ and thereafter compute the \emph{maximum likelihood estimate (MLE)} of the parameter $\Theta_*$ on the subset $\cS$, i.e., 
\begin{gather}
    \min_{\Theta}\;\; \cL_{\cS}(\Theta),\nonumber \\
    \text{where}\, \cL_{\cS}(\Theta) := -\frac{1}{n}\sum_{i \in \cS} \sum_{j = 1}^{M_i} \log p(y_{i, j} \mid \x_{i, j}; \Theta).
  \label{eq:loglik_subset}
\end{gather} 
When applied to an LLM fine-tuning, we use the linearized model (with the pre-logit embeddings) only to select the subset $\cS$ and instead of computing an MLE estimate $\hat{\Theta}$, we train all the parameters of the network.

\section{Algorithm}
\label{sec:algo}
The \emph{Fisher information matrix} \cite{fisher22mathematical} corresponds to the Hessian of the negative log likelihood with respect to $\Theta$ and is given by
\begin{align}
  \nabla^2 \cL_\cS(\Theta)
  = -\frac{1}{n}\sum_{i \in S} \sum_{j = 1}^{M_i}
  \nabla^2 \log p(y_{i, j} \mid \x_{i, j}; \Theta)\,.
  \label{eq:hessian}
\end{align}
The Hessian $\nabla^2 \cL_\cS(\Theta)$ can be used to derive the covariance matrix of the MLE of $\cL_\cS(\Theta)$. Therefore, it can be used for both uncertainty quantification and information gathering. Specifically, a high-probability confidence interval on model parameters $\Theta_*$ can be typically derived using $\nabla^2 \cL_\cS(\Theta_*)$ \citep{abbasi-yadkori11improved,lattimore19bandit}. In this work, we optimize $\nabla^2 \cL_\cS(\Theta_*)$ by maximizing all of its eigenvalues with respect to $\cS$, which can be tractably approached as $\log\det(\nabla^2 \cL_\cS(\Theta_*))$ maximization.

This problem is hard for three reasons. First, $\nabla^2 \cL_{\cS}(\Theta_*)$ is a $d L \times d L$ times matrix. Therefore, for practical values of $d \approx 1000$ and $L > 100\,000$, it is computationally costly to optimize it. Second, the exact maximization is impossible because $\Theta_*$ is unknown. To address these two challenges, we derive a lower bound on $\log\det(\nabla^2 \cL_{\cS}(\Theta))$ that only involves $d \times d$ matrices and is $\Theta$-independent. We present the lower bound in the following lemma.

\begin{algorithm}[t!]
  \caption{Greedy Optimal Design \\ for Autoregressive Models.}
  \label{alg:Greedy-OD}
  \begin{algorithmic}[1]
    \State \textbf{Input:} Sentences $\{\x_i = (\x_{i, j})_{j = 1}^{N_i}\}_{i = 1}^{N}$
    \State Design matrix $V \gets I_d$
    \State Selected sentences $\cS \gets \emptyset$
    \For{$t = 1, \ldots, n$}
      \State $\displaystyle
      k \gets \argmax_{i \in [N] \setminus \cS} \,
      \log\det\left(V + \sum_{j = 1}^{M_i} \x_{i, j} \x_{i, j}^\top\right)$
      \State $\cS \gets \cS + \{k\}$
      \State $\displaystyle
      V \gets V + \sum_{j = 1}^{M_k} \x_{k, j} \x_{k, j}^\top$
    \EndFor
    \State \textbf{Output:} $\cS$
  \end{algorithmic}
\end{algorithm}

\begin{lemma}
Consider the loss function described in \eqref{eq:loglik_subset}. Then the Hessian of the loss is given by
\begin{align*}
  \nabla^2 \cL_{\cS}(\Theta) &= \frac{1}{n} \sum_{i \in \cS} \sum_{j = 1}^{M_i} \Big[\mathbf{diag}(p(y_{i,j}| \x_{i,j};\Theta)) \\
  & - p(y_{i,j}| \x_{i,j};\Theta)p(y_{i,j}| \x_{i,j};\Theta)^\top \Big] \otimes \x_{i,j}\x_{i,j}^\top,
\end{align*} 
where $\otimes$ is the tensor product. Moreover, if
\begin{align*}
  \mathbf{diag}(p(y_{i,j}| \x_{i,j};\Theta)) - p(y_{i,j}| \x_{i,j};\Theta)p(y_{i,j}| \x_{i,j};\Theta)^\top \!\! \succeq \! \gamma
\end{align*}
holds for some $\gamma > 0$, then
\begin{align*}
  \log\det (\nabla^2\cL_{\cS}(\Theta)) \geq d\log\det \Big(\frac{\gamma}{n} \sum_{i \in S} \sum_{j = 1}^{M_i} \x_{i,j}\x_{i,j}^\top\Big)\,.
\end{align*}
\label{lemma:log-det-lower-bound}
\end{lemma}
\begin{proof}
The lemma is proved in \cref{sec:log-det-lower-bound proof}.
\end{proof}

Therefore, instead of maximizing $\log\det(\nabla^2 \cL_{\cS}(\Theta))$, we can maximize $\log \det(\sum_{i \in \cS} \sum_{j = 1}^{M_i} \x_{i,j} \x_{i,j}^\top)$. The last challenge is that we have a combinatorial optimization problem, choose a subset of $n$ sentences out of $N$. Since $\log\det$ is a monotone submodular function, we solve this problem greedily \citep{nemhauser78approximation}.

\subsection{Greedy Optimal Design}

Our greedy algorithm is presented in \cref{alg:Greedy-OD}. We refer to the optimized Hessian as a \emph{design matrix}, because the matrix is used to design the set of chosen sentences. The design matrix is initialized at $V = I_d$ (line 2) and the subset of selected sentences is initialized at $\cS = \emptyset$ (line 3). In each step $t \in [n]$, the algorithm selects the sentence, from the remaining sentences $[N] \setminus \cS$, that maximizes $\log\det$ of the design matrix of the previously chosen sentences (line 5). This sentence has the highest information gain. Intuitively, it contains the most diverse embeddings $\x_{i,j}^\top$ since $\log\det(V)$ can be viewed as the logarithm of the volume of an ellipsoid represented by $V$, and this is maximized when the lengths of all its axes increase equally. After the sentence is chosen, it is added to the current subset of sentences $\cS$ (line 6) and $\sum_{j=1}^{M_i} \x_{k,j} \x_{k,j}^\top$ is added to the design matrix $V$ (line 7).

Note that \cref{alg:Greedy-OD} selects one sentence at a time and each such iteration involves computing $\log\det$ of all remaining sentences (line 5). Such an implementation is clearly impractical. In \cref{sec:fast greedy}, we present a computationally faster algorithm that takes advantage of the submodularity of $\log \det$ and parallelism to produce the same subset of sentences as in \cref{alg:Greedy-OD}.

We are concerned with two variants of \cref{alg:Greedy-OD} in this work. In \cref{sec:proof}, we analyze it in the idealized setting where the pre-logit layer of the LLM is treated as a fixed feature vector. After \cref{alg:Greedy-OD} collects $n$ samples, we use maximum likelihood estimation to compute the estimated model parameters
\begin{align}
  \hat{\Theta}
  = \argmin_{\Theta} \cL_\cS(\Theta)\,,
  \label{eq:theta_hat}
\end{align} 
where $\cL_{\cS}(\Theta)$ is defined in \eqref{eq:loglik_subset}. We argue that $\hat{\Theta}$ approaches $\Theta_*$ as the sample size $n$ increases.

When applied to LLMs, \cref{alg:Greedy-OD} collects $n$ sentences that are used to fine-tune an actual LLM. The embedding of the $j$-th token in sentence $i$ is the output of the pre-logit layer of the LLM, denoted by $\x_{i, j}$.

\subsection{Fast Greedy Optimal Design}
\label{sec:fast greedy}
\begin{algorithm}[t!]
  \caption{\FisherSFT: Fast Implementation of Algorithm~\ref{alg:Greedy-OD}}
  \label{alg:Greedy-OD-fast}
  \begin{algorithmic}[1]
    \State \textbf{Input:} Sentences $\{\x_i = (\x_{i, j})_{j = 1}^{N_i}\}_{i = 1}^{N}$, batch size $B$
    \State Design matrix $V \gets I_d$
    \State Selected sentences $\cS \gets \emptyset$
    \State Cached information gains $g \gets \infty_N$
    \For{$t = 1, \ldots, n$}
      \State $g_{\max} \gets 0$
      \For{$b = 1, \ldots, N / B$}
        \State $\cB \gets \{(b - 1) B + 1, \dots, b B\}$
        \ForAll{$i \in \cB$}
          \If{$g_i > g_{\max}$}
            \State $\displaystyle
            g_i \gets \log\det\left(V +
            \sum_{j = 1}^{M_i} \x_{i, j} \x_{i, j}^\top\right) -$
            \Statex $\hspace{1.16in} \log\det(V)$
          \EndIf
        \EndFor
        \State $g_{\max} \gets \max \, \{g_i\}_{i \in \cB} + \{g_{\max}\}$
      \EndFor
      \State $k \gets \argmax_{i \in [N] \setminus \cS} g_i$
      \State $\cS \gets \cS + \{k\}$
      \State $\displaystyle
      V \gets V + \sum_{j = 1}^{M_k} \x_{k, j} \x_{k, j}^\top$
    \EndFor
    \State \textbf{Output:} $\cS$
  \end{algorithmic}
\end{algorithm}

Now we present a more computationally-efficient variant of \cref{alg:Greedy-OD} that exploits the submodularity of $\log\det$ and parallelism (\cref{alg:Greedy-OD-fast}). Simply put, we implement line 5 in \cref{alg:Greedy-OD} more efficiently, which is correspond to line 13 in \cref{alg:Greedy-OD-fast}.

The key idea is to cache information gains, where $g_i$ is the cached information gain for sentence $i \in [N]$. The gains are initialized as $g_i \gets \infty$ (line 4), updated in line 11, and we act greedily with respect to them in line 13. If the gains were always updated, note that line 13 is equivalent to line 5 in \cref{alg:Greedy-OD}, because the matrix $V$ is a constant in step $t$.

The key insight to efficient updates is that $\log\det$ is monotone and submodular. Therefore, the gains cannot increase as $V$ is updated and thus do not have to be recomputed when they are smaller than the highest tracked gain $g_{\max}$ at any step $t \in [n]$. We exploit this in line 10 and update $g_{\max}$ in line 12. Finally, we update $g_i$ in batches of size $B$ (line 9). This can be done in parallel and results in an additional $O(B)$ speedup. We use this implementation in our experiments and refer to it as \FisherSFT.

\subsection{Proof of \cref{lemma:log-det-lower-bound}}
\label{sec:log-det-lower-bound proof}

    In Section~\ref{sec:grad_Hessian} \cref{prop:grad_hessian} we show that
    \begin{align*}
    \nabla^2 \cL_{\cS}(\Theta) &= \frac{1}{n} \sum_{i \in \cS} \sum_{j = 1}^{M_i} \Big(\mathbf{diag}(p(y_{i,j}| \x_{i,j};\Theta)) \\
    & - p(y_{i,j}| \x_{i,j};\Theta)p(y_{i,j}| \x_{i,j};\Theta)^\top \Big) \otimes \x_{i,j}\x_{i,j}^\top
\end{align*}
Now suppose for some $\gamma > 0$, $\mathbf{diag}(p(y_{i,j}| \x_{i,j};\Theta)) - p(y_{i,j}| \x_{i,j};\Theta)p(y_{i,j}| \x_{i,j};\Theta)^\top \!\succeq \!\gamma$. Then 
\begin{align*}
    \nabla^2 \cL_{\cS}(\Theta) &\succeq \frac{1}{n} \sum_{i \in \cS} \sum_{j = 1}^{M_i} \gamma I_L \otimes \x_{i,j}\x_{i,j}^\top
\end{align*}
where $I_{L}$ is the $L$ dimensional identity matrix. Therefore we have
    $$\det(\nabla^2\cL_{\cS}(\Theta)) \geq  \det \Big(I_{L} \otimes \frac{\gamma}{n} \sum_{i \in S} \sum_{j = 1}^{M_i} \x_{i,j}\x_{i,j}^\top\Big)$$
    Now using the fact that for $A \in \R^{p\times p}$ and $B \in \R^{q\times q}$, $\det(A \otimes B) = \det(A)^p \det(B)^q$ (See Proposition 7.1.11. in \cite{bernstein2009matrix}) we have
    \begin{align*}
        \det (\nabla^2\cL_{\cS}(\Theta)) &\geq \det(I_L)^L\det \Big(\frac{\gamma}{n} \sum_{i \in S} \sum_{j = 1}^{M_i} \x_{i,j}\x_{i,j}^\top\Big)^d \\
        &\geq  \det \Big(\frac{\gamma}{n} \sum_{i \in S} \sum_{j = 1}^{M_i} \x_{i,j}\x_{i,j}^\top\Big)^d
    \end{align*} 
    Finally
    \begin{align*}
        \log \det (\nabla^2\cL_{\cS}(\Theta)) \geq d \log \det \Big(\frac{\gamma}{n} \sum_{i \in S} \sum_{j = 1}^{M_i} \x_{i,j}\x_{i,j}^\top\Big),
    \end{align*}
    completes the proof.

\section{Error Bound}
\label{sec:proof}
Our main \cref{thm:main_error_bound} provides a $O(1 / \sqrt{n})$ bound on the \emph{maximum prediction error} of the estimated parameter $\hat{\Theta}$ constructed using the samples generated by \cref{alg:Greedy-OD}. The \emph{maximum prediction error} is given by the following expression
\begin{align*}
    \max_{i \in [N]} \sum_{j=1}^{M_i} \|\Theta_*^\top\x_{i,j} - \hat{\Theta}^\top\x_{i,j}\|_2.
\end{align*}
Note that the \emph{maximum prediction error} measures $\|\cdot\|_2$, i.e., it is the sum of prediction errors over the whole vocabulary, at the $j$-th token in sentence $i$, and therefore captures the error across all the $L$ words.
We make the following assumption on the feature vectors and the unknown parameters.
\begin{assumption}
    \label{asmp:feature}
    Assume that $\forall i\in [N], j \in [M_i],$ $\|\x_{i,j}\| \leq 1$. Further 
    we assume that the true model parameter $\Theta_*$ satisfies $\Theta_* \in \cB$ where
    \begin{align*}
        \cB := \{\Theta = (\theta_\ell)_{\ell \in [L]}: \theta_{\ell} \in \R^{d}, \|\theta_{\ell}\|_{2} \leq 1, \Theta \1 = 0\}
    \end{align*}
\end{assumption} 

Further we make a diversity assumption on our dataset.
Given any arbitrary subset $\cS \subseteq \cD$, we define
\begin{align}
  {\bar{\Sigma}}_\cS
  &= \sigma_0 \bm{I_d} + \sum_{i \in \cS} \sum_{j = 1}^{M_i} \x_{i,j} \x_{i,j}^\top\;,
  \label{eq:Xi_S}
\end{align} 
\begin{assumption}
\label{ass:diverse-feature-vectors} There exists a constant $\kappa \geq 1$ such that
\begin{align*}
\log\det(I_d &+ \sum_{j=1}^{M_i} \Sigma_{t-1}^{-1/2} x_{i,j}x_{i,j}^\top \Sigma_{t-1}^{-1/2}) \\
& \leq \kappa \log\det(I_d + \sum_{j=1}^{M_{I_t}} \Sigma_{t-1}^{-1/2} x_{I_t,j}x_{I_t,j}^\top \Sigma_{t-1}^{-1/2})
\end{align*}
holds for any $i \in \mathcal{S}_{t-1}$ and $t \in [n]$.
\end{assumption}
\begin{theorem}
\label{thm:main_error_bound}
    Suppose \cref{asmp:feature} and \cref{ass:diverse-feature-vectors} hold. Then for any $\delta>0$, under the softmax model in \eqref{eq:softmax}, with probability $1-\delta$, the maximum prediction error of $\hat{\Theta}$ can be bounded as follows:
    \begin{align*}
    &\max_{i\in [N]} \sum_{j=1}^{M_i} \|\Theta_*^\top\x_{i,j} - \hat{\Theta}^\top\x_{i,j}\|_2 \\
    &\leq C M e^2 L \sqrt{\frac{\sigma_0^{-2} \log\left(1 + \frac{\sigma_0^{-2} n M}{d}\right)}
  {\log(1 + \sigma_0^{-2})}} \sqrt{\frac{d\kappa(d + \log (L/\delta))}{n}}\,.
\end{align*}
where $C>0$ is some global constant.
\end{theorem}

\subsection{Proof Sketch}

Suppose $\cS$ be the subset of $n$ sentences produced by \tokenod. With $\hat{\Theta} = (\hat{\theta}_{\ell})_{\ell \in [L]}$ and ${\Theta_*} = ({\theta}_{\ell}^*)_{\ell \in [L]}$ we can decompose the error as follows:
\begin{align}
    \max_{i \in [N]} & \sum_{j=1}^{M_i} \|\hat{\Theta}^\top \x_{i,j} - \Theta_*^\top \x_{i,j}\|_2 \\
    &\leq \max_{i \in [N]} \sum_{j=1}^{M_i} \sum_{\ell \in [L]} |(\hat{\theta}_{\ell} - \theta_{\ell}^{*})^{\top} \x_{i,j}| \nonumber \\
    &
    \leq \max_{i \in [N]} \sum_{j=1}^{M_i} \sum_{\ell \in [L]} \|\hat{\theta}_{\ell} - \theta_{\ell}^{*}\|_{\bar{\Sigma}_{\cS}} \|\x_{i,j}\|_{\bar{\Sigma}_{\cS}^{-1}} \\
    & \leq \underbrace{\Big( \sum_{\ell \in [L]} \|\hat{\theta}_{\ell} - \theta_{\ell}^{*}\|_{\bar{\Sigma}_{\cS}} \Big)}_{\text{I}}\;\underbrace{\max_{i \in [N]} \sum_{j=1}^{M_i} \|\x_{i,j}\|_{\bar{\Sigma}_{\cS}^{-1}}}_{\text{II}}
    \label{eq:errorDecomposition}
\end{align}
where $\displaystyle \bar{\Sigma}_{\cS} = \sigma_0^2 \bm{I} + \sum_{i \in \cS} \sum_{j = 1}^{M_i} \x_{i,j}\x_{i,j}^\top$. Term I corresponds to the error between the true parameter ${\Theta_{*}}$ and the MLE estimate $\hat{\Theta}$ while term II measures the maximum curvature of $\bar{\Sigma}_{\cS}$.

Let us first consider term II. Under \cref{ass:diverse-feature-vectors} we show that term II is bounded as follows.
\begin{restatable}{lemma}{curvatureProp}
        Suppose \cref{ass:diverse-feature-vectors} holds and $\cS$ be the subset of sentences produced by Algorithm~\ref{alg:Greedy-OD}, $\displaystyle \bar{\Sigma}_{\cS} = \sum_{i \in \cS} \sum_{j = 1}^{M_i} \x_{i,j}\x_{i,j}^\top$ is the covariance matrix constructed using the samples in $\cS$ and $M = \max_{i \in [N]} M_i$. Then
        \begin{align}
        \label{eq:curvature}
            \max_{i \in [N]} \sum_{j=1}^{M_i} \|{\x}_{i,j}\|^{2}_{\bar{\Sigma}_{\cS}^{-1}} \leq \frac{\sigma_0^{-2} \log\left(1 + \frac{\sigma_0^{-2} n M}{d}\right)}
  {\log(1 + \sigma_0^{-2})} \frac{\kappa d M}{n}.
        \end{align} 
        \label{prop:curvatureLemma}
\end{restatable}
Next we need to control term I in \eqref{eq:errorDecomposition}. To do this we relate term I to the difference between the loss and its first order approximation as below
\begin{align}
     \cL_{\cS}(\hat{\Theta}) &- \cL_{\cS}(\Theta^*) - \langle \nabla \cL_{\cS}(\Theta_*), \hat{\Theta} - \Theta^* \rangle \nonumber \\
     & \overset{(a)}{\leq} - \langle \nabla \cL_{\cS}(\Theta_*), \hat{\Theta} - \Theta^* \rangle \nonumber \\
     &= - \sum_{\ell = 1}^{L}  \nabla_{\ell} \cL_{\cS}(\Theta_*)^\top (\hat{\theta}_{\ell} - \theta_{\ell}^{*}) \nonumber\\
     &\overset{(b)}{\leq} \sum_{\ell = 1}^{L}  \big\|\nabla_{\ell} \cL_{\cS}(\Theta_*)\big\|_{\bar{\Sigma}_{\cS}^{-1}} \|\hat{\theta}_{\ell} - \theta_{\ell}^{*}\|_{\bar{\Sigma}_{\cS}}
     \label{eq:loss_upperbound}
\end{align}
where the dot product between matrices $A$ and $B$ is defined as $\langle A, B \rangle = \sum_{i,j} A_{i,j} B_{i,j}$.
Inequality $(a)$ follows from $\cL_{\cS}(\hat{\Theta}) \leq \cL_{\cS}({\Theta}^*)$ and $(b)$ follows from Cauchy Schwarz inequality.
Next we lower bound $\cL_{\cS}(\hat{\Theta}) - \cL_{\cS}(\Theta^*) - \langle \nabla \cL_{\cS}(\Theta_*), \hat{\Theta} - \Theta^* \rangle$ by showing that the loss is strongly convex at $\Theta^*$ in the following lemma.
\begin{restatable}{lemma}{stronglyConvexLemma}
    Suppose \cref{asmp:feature} holds and $\hat{\Theta}$ be the MLE solution as in \eqref{eq:theta_hat} such that $\hat{\Theta} \in \cB$. Then, there exists some $\alpha < 1$ such that
    \begin{align*}
            \cL_{\cS}(\hat{\Theta}) &- \cL_{\cS}(\Theta^*) - \langle \nabla \cL_{\cS}(\Theta_*), \hat{\Theta} - \Theta^* \rangle \\
            & \geq \frac{e^{-2\alpha}}{L}\bigg(\sum_{\ell} \|\hat{\theta}_{\ell} - \theta^*_{\ell}\|_{\bar{\Sigma}_{\cS}}\bigg)^2.
        \end{align*}
        \label{lemma:stronglyConvexityLemma}
\end{restatable}

Using Lemma~\ref{lemma:stronglyConvexityLemma} and \eqref{eq:loss_upperbound} we have
\begin{align*}
    \frac{e^{-2\alpha}}{L} & \bigg(\sum_{\ell} \|\hat{\theta}_{\ell} - \theta^*_{\ell}\|_{\bar{\Sigma}_{\cS}}\bigg)^2  \leq \sum_{\ell = 1}^{L}  \big\|\nabla_{\ell} \cL_{\cS}(\Theta_*)\big\|_{\bar{\Sigma}_{\cS}^{-1}} \|\hat{\theta}_{\ell} - \theta^*_{\ell}\|_{\bar{\Sigma}_{\cS}}\\
    &\leq \sup_{\ell \in [L]} \big\|\nabla_{\ell} \cL_{\cS}(\Theta_*)\big\|_{\bar{\Sigma}_{\cS}^{-1}} \bigg(\sum_{\ell} \|\hat{\theta}_{\ell} - \theta^*_{\ell}\|_{\bar{\Sigma}_{\cS}}\bigg)
\end{align*}
and therefore
\begin{align*}
    \bigg(\sum_{\ell} \|\hat{\theta}_{\ell} - \theta^*_{\ell}\|_{\bar{\Sigma}_{\cS}}\bigg) &\leq e^{2\alpha} L \sup_{\ell \in [L]} \big\|\nabla_{\ell} \cL_{\cS}(\Theta_*)\big\|_{\bar{\Sigma}_{\cS}^{-1}}\\
    & \leq e^{2} L \sup_{\ell \in [L]} \big\|\nabla_{\ell} \cL_{\cS}(\Theta_*)\big\|_{\bar{\Sigma}_{\cS}^{-1}}
\end{align*}
The next lemma bounds $\sup_{\ell \in [L]} \big\|\nabla_{\ell} \cL_{\cS}(\Theta_*)\big\|_{\bar{\Sigma}_{\cS}^{-1}}$.

\begin{restatable}{lemma}{gradientBound}
    With probability $1-\delta$ the gradient of the loss satisfies the following bound:
    \begin{align}
    \label{eq:gradient_norm}
        \sup_{\ell \in [L]} \big\|\nabla_{\ell} \cL_{\cS}(\Theta_*)\big\|_{\bar{\Sigma}_{\cS}^{-1}} \leq C \sqrt{{{d + \log(L/\delta)}}}
    \end{align}
    where $C>0$ is some global constant.
\end{restatable}
Combining \eqref{eq:gradient_norm}, \eqref{eq:loss_upperbound}, \eqref{eq:curvature} and \eqref{eq:errorDecomposition} we have with probability $1-\delta$
\begin{align*}
    &\max_{i \in [N]}  \sum_{j=1}^{M_i} \|\Theta_*^\top\x_{i,j} - \hat{\Theta}^\top\x_{i,j}\|_2 \\
    & \leq C M e^2 L \sqrt{\frac{\sigma_0^{-2} \log\left(1 + \frac{\sigma_0^2 n M}{d}\right)}
  {\log(1 + \sigma_0^2)}} \sqrt{\frac{d\kappa(d + \log (L/\delta))}{n}},
\end{align*} 
where $C>0$ is some constant, thus completing the proof.

\section{Experiments}
\label{sec:experiments}
\begin{figure*}[t]
  \centering
  \subfigure{\includegraphics[width=0.40\textwidth]{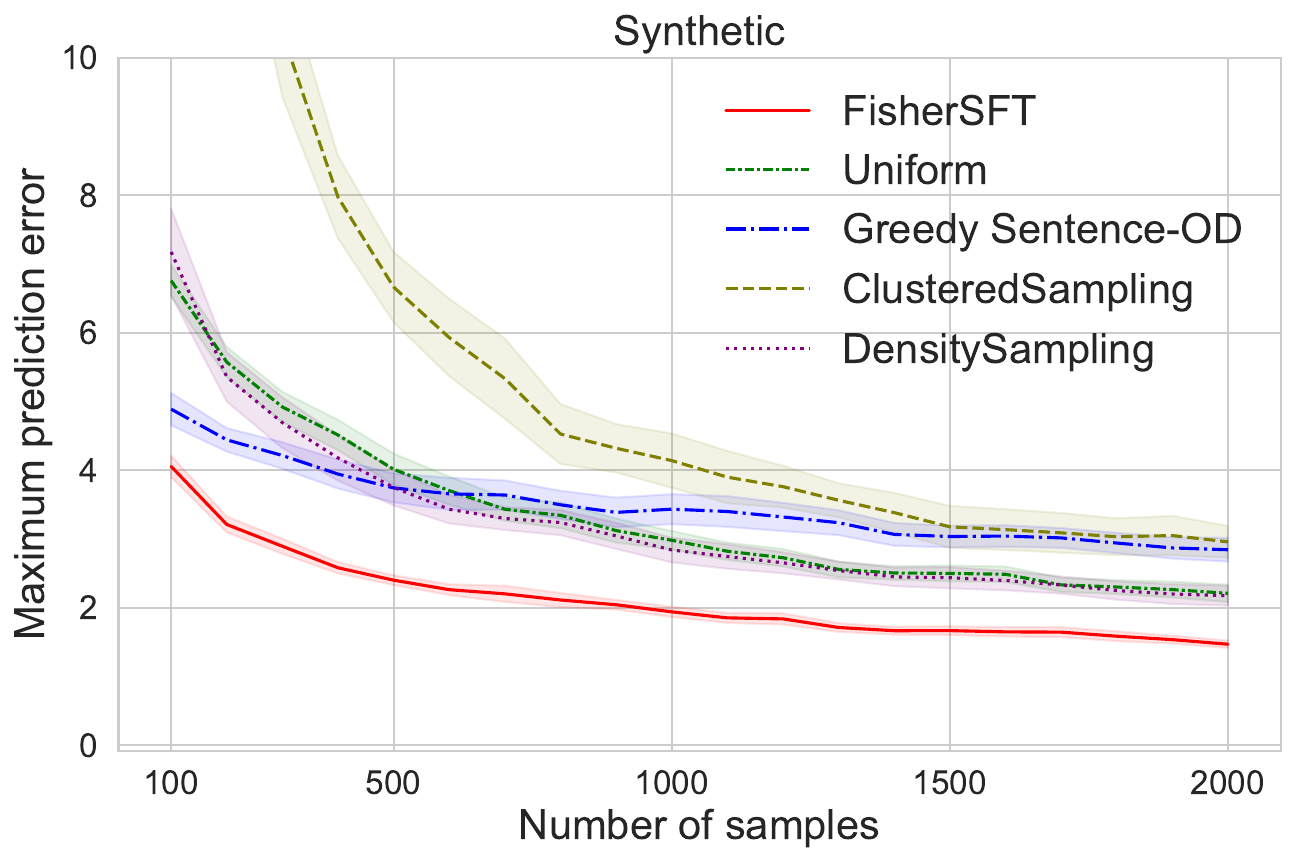}} 
  \subfigure{\includegraphics[width=0.40\textwidth]{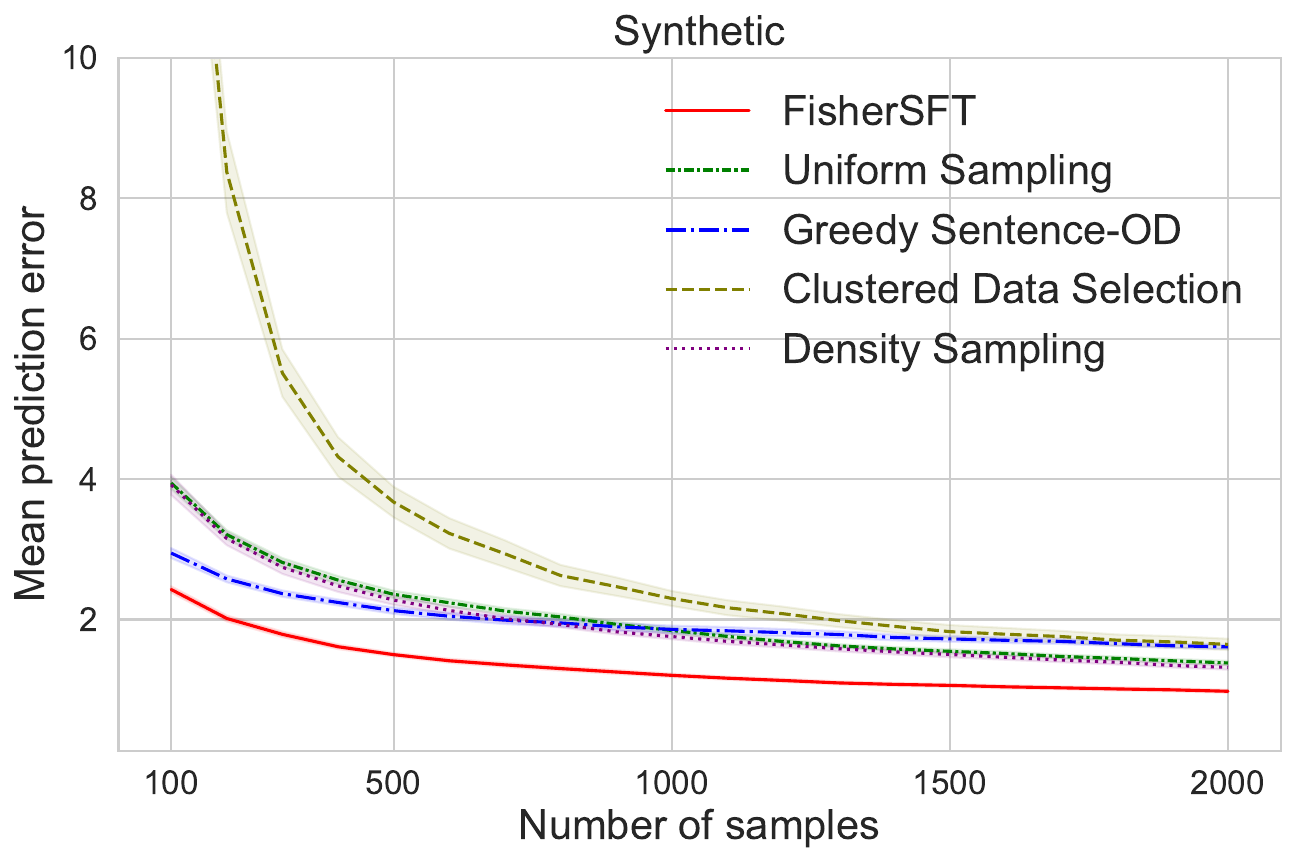}} 
  \caption{Comparison of maximum and mean prediction errors on synthetic token vectors. The x axis shows the number of sentences selected to train the model. The y axis shows the corresponding error averaged over $20$ runs.} 
  \label{fig:1}
\end{figure*}

\begin{figure*}[t]
  \centering
  \subfigure{\includegraphics[width=0.40\textwidth]{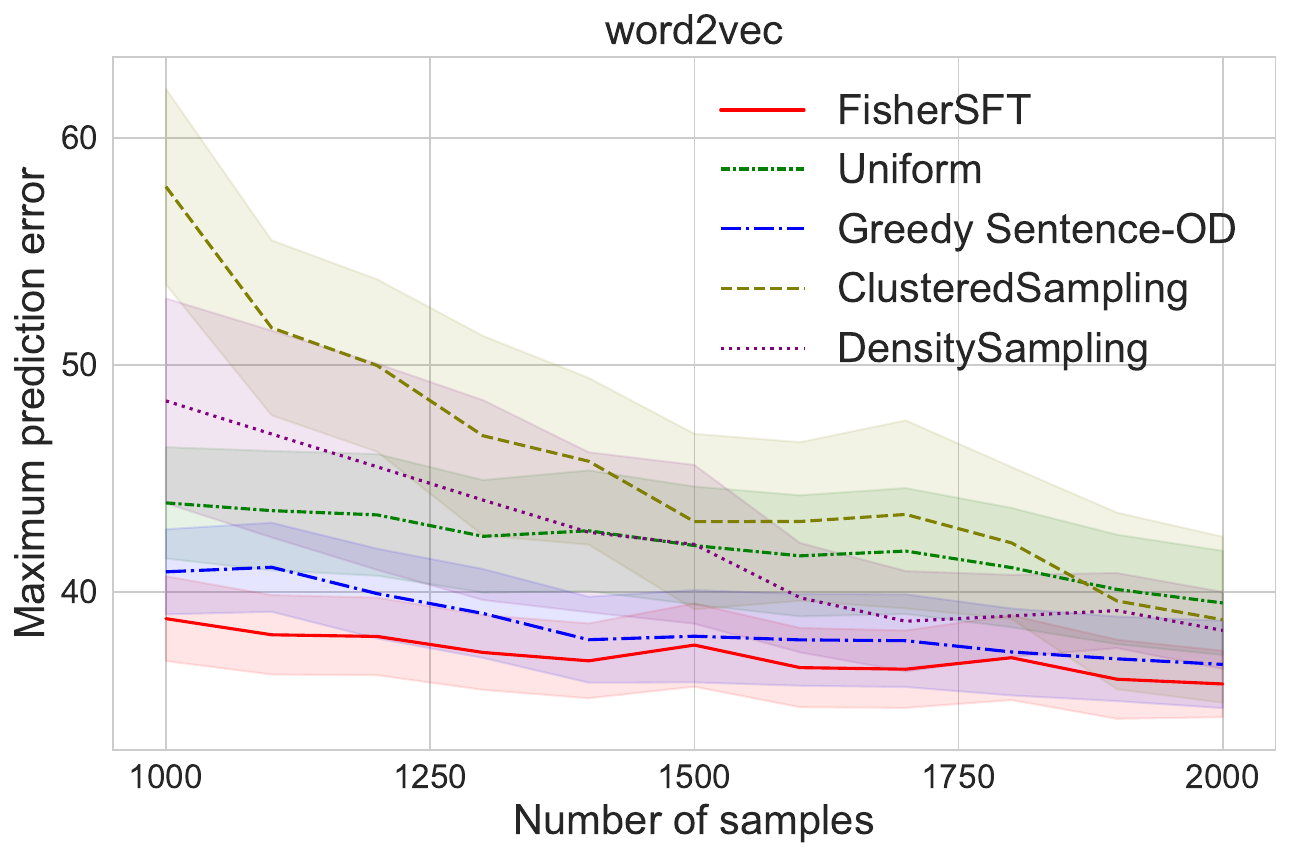}} 
  \subfigure{\includegraphics[width=0.40\textwidth]{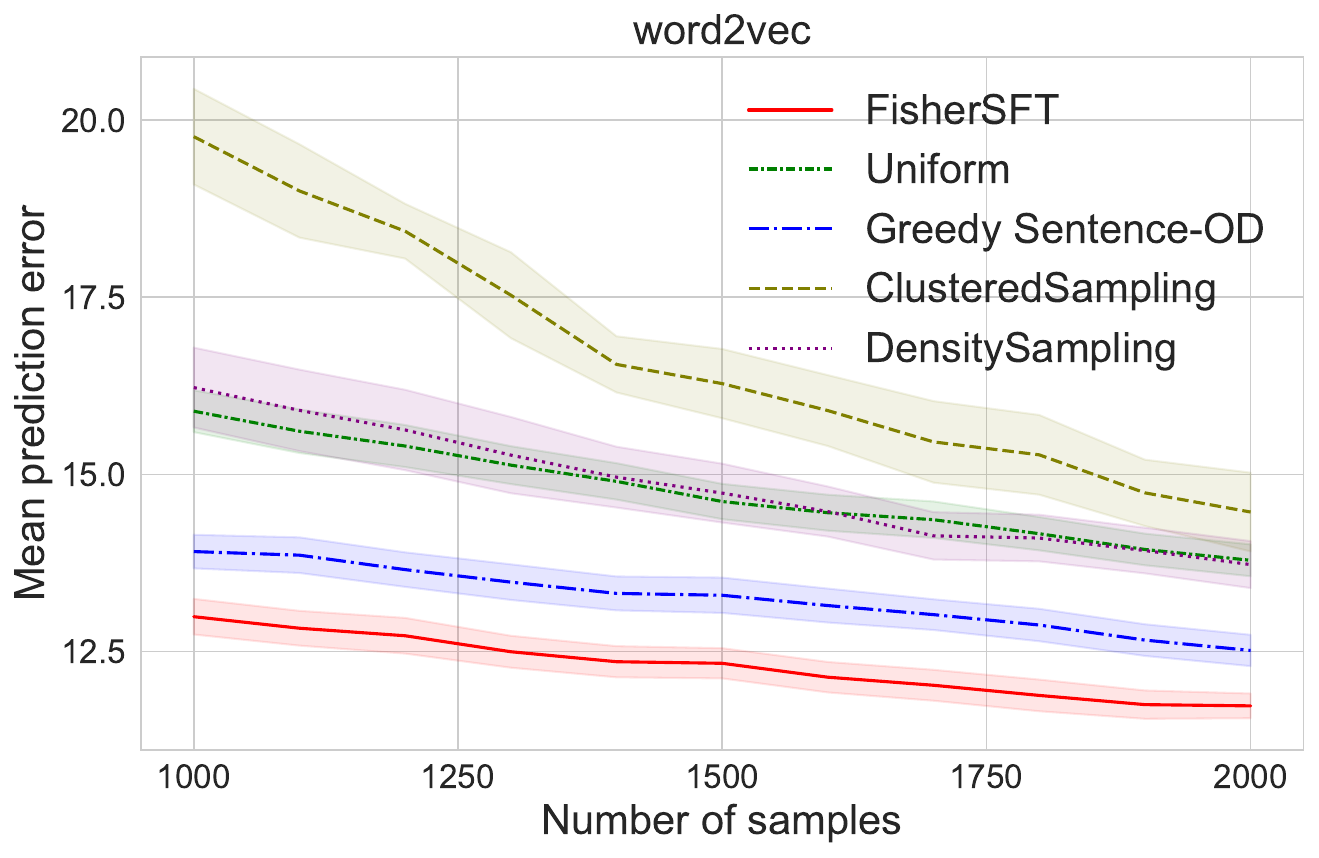}} 
  \caption{Comparison of maximum and mean prediction errors on word2vec token vectors. The x axis shows the number of sentences selected to train the model. The y axis shows the corresponding error averaged over $20$ runs.}
  \label{fig:2}
\end{figure*}

In this section, we empirically evaluate our algorithm and compare it to baselines. We experiment with a synthetic autoregressive prediction task in \cref{subsec:simulated}, with pre-trained word embeddings in \cref{subsec:word2vec}, and with GPT-2 on text dataset in \cref{subsec:gpt2}.

\subsection{Synthetic Experiments}
\label{subsec:simulated}

\begin{figure*}[t]
  \centering
  \subfigure{\includegraphics[width=\textwidth]{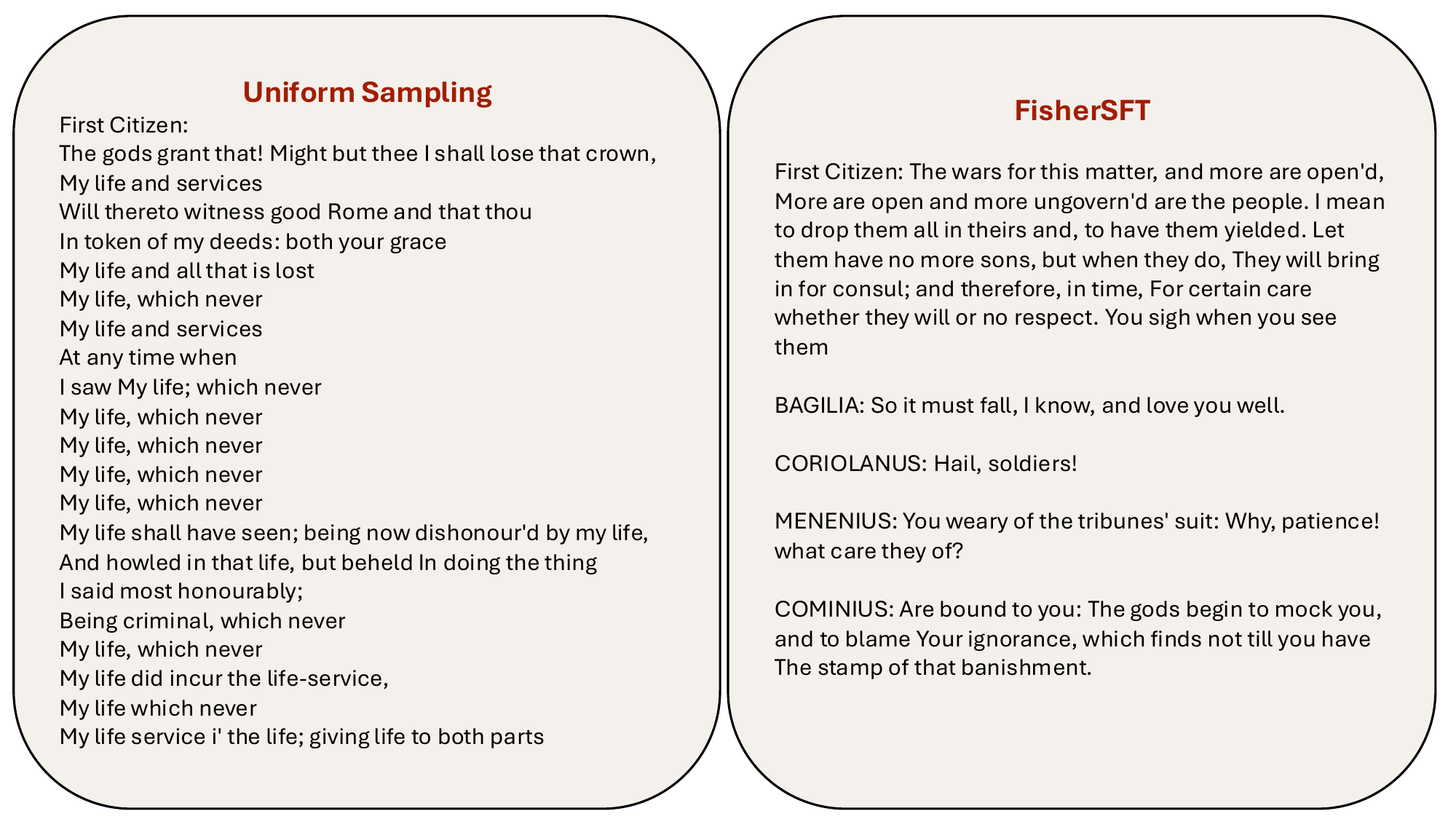}}
  \caption{Text generated by fine-tuned GPT-2 models on sentences selected by \uniform and \FisherSFT. The latter is more coherent.}
  \label{fig:3}
\end{figure*}

We start with a simplified setup where each token $\ell \in [L]$ is associated with a vector sampled from a standard normal distribution, $\x_\ell \sim \cN(\mathbf{0}, I_d)$. The number of tokens is $L = 20$ and $d = 10$. The first token in each sentence is sampled uniformly at random from all tokens. Each next token is sampled the softmax model in \eqref{eq:softmax}, where all entries of $\Theta_*$ are sampled i.i.d.\ from $\cN(0, 1)$.

We compare \FisherSFT to several baselines. \uniform selects sentences uniformly at random. \sentenceod selects sentences greedily by maximizing $\log\det$ of a sentence-level Fisher information matrix. We construct sentence embeddings by summing up the token vectors in the sentences, $\x_i = \sum_{j=1}^{M_i}\x_{i,j}$, where $\x_i$ denotes the vector for sentence $i \in [N]$. \density \citep{sachdeva2024how} uses inverse propensity sampling to select sentences based on a score computed by a kernel density estimate. \cluster \citep{axiotis2024data} clusters the sentence embeddings using $k$-means clustering and then samples them proportionally to their distance to the closest mean plus the mean’s loss. See \cref{sec:related} for more details on the baselines.

All methods are evaluated as follows. After they choose the set of sentences $\cS$, $\hat{\Theta}$ is estimated using multinomial logistic regression. We evaluate the methods by two metrics: \textit{maximum prediction error}
\begin{align*}
  \displaystyle \mathcal{E}_{\max}(n) = \max_{i \in [N]} \sum_{j=1}^{M_i} \|\Theta_*^\top\x_{i,j} - \hat{\Theta}^\top\x_{i,j}\|_2
\end{align*}
and \textit{mean prediction error}
\begin{align*}
  \displaystyle \mathcal{E}_{mean}(n) = \frac{1}N{}\sum_{i \in [N]} \sum_{j=1}^{M_i} \|\Theta_*^\top\x_{i,j} - \hat{\Theta}^\top\x_{i,j}\|_2\,.
\end{align*}
The maximum error measures the performance on the most challenging sentence, while the mean error measures the average performance on all sentences. Note that we bound the maximum error of \FisherSFT in \cref{thm:main_error_bound}.

We plot the errors for our synthetic problem in \cref{fig:1} and observe that \FisherSFT performs better than all baselines in both metrics. In most regimes, \FisherSFT is much more sample efficient than the best baseline. As an example, the lowest maximum error of the best baseline, which is attained at $n = 2\,000$, is attained by \FisherSFT at $n = 1\,000$.

\subsection{Pre-trained Word Embeddings}
\label{subsec:word2vec}

The main difference in this experiment comparing to \cref{subsec:simulated} is that we use pre-trained \texttt{word2vec} embeddings \cite{Mikolov2013EfficientEO} of dimension $300$. We randomly select $L=20$ words from the \texttt{word2vec} vocabulary and project their embeddings randomly to $d = 10$ dimensions. The vector associated with token $\ell \in [L]$ is $\x_\ell$. The rest is the same as in \cref{subsec:simulated}. We report the maximum and mean prediction errors of all methods in \cref{fig:2}. Again \FisherSFT outperforms all baselines from \cref{subsec:simulated} in both metrics. As an example, the lowest mean error of the best baseline, which is attained at $n = 2\,000$, is attained by \FisherSFT at $n = 1\,500$.

\subsection{Experiments with GPT-2}
\label{subsec:gpt2}
\begin{table*}[t]
  \centering
  \caption{Comparison of \FisherSFT against various baseline data sub-sampling strategies when fine-tuning GPT-2 on the Shakespeare dataset.
  Entries show the fraction of times \FisherSFT was preferred over the corresponding baseline. All the fractions being greater than 0.5 implies that \FisherSFT outperforms all the baselines}
  \label{tab:table_combined}
  \label{tab:combined}
  \begin{tabular}{@{}lccccccc@{}}
    \toprule
    \textbf{Sampling strategy} & \multicolumn{7}{c}{\textbf{Number of sentences sub-sampled for finetuning}} \\ 
    \cmidrule(lr){1-8}
    {\FisherSFT vs baseline} & 
    {100} & 
    {200} & 
    {500} & 
    {1000} & 
    {2000} & 
    {5000} \\
    \midrule
    vs \uniform & 0.80 & 0.56 &	0.60 &	0.59 &	0.64 &	0.74 \\
    vs \density & 0.61 & 0.66 &	0.68 &	0.62 &	0.54 &	0.84 \\
    vs \askLLM  & 0.59 & 0.52 &	0.68 &	0.59 &	0.68 &	0.74 \\
    \bottomrule
  \end{tabular}
\end{table*}

\textbf{Model and Datasets:} We experiment with a tiny-Shakespeare corpus \citep{shakespeare}. Our corpus $\cD$ is its subset of $10\,000$ sentences. We actively select $n \in \{50, 100, 200, 500, 1\,000, 2\,000, 5\,000\}$ sentences and then fine-tune a GPT-2 model \citep{radford2019language} on them. We use the Hugging Face implementation of GPT-2 \citep{wolf2020huggingfacestransformersstateoftheartnatural}. 

\textbf{Baselines:} We experiment with \density \citep{sachdeva2024how}, \askLLM \citep{sachdeva2024how}, and \uniform sampling baselines. See \cref{sec:related} for a detailed description of the baselines. We choose \density because it outperforms other methods on language model fine-tuning tasks \citep{sachdeva2024how}.

\textbf{LLM-based Evaluation:} Unlike in \cref{subsec:simulated,subsec:word2vec}, the ground truth model parameter is not available, and thus the maximum and mean prediction errors cannot be computed. Therefore, we judge the quality of the fine-tuned model using a state-of-the-art LLM. Specifically, we generate new text using the fine-tuned model by prompting it with $100$ different phrases from the original dataset. Then we compare the generated text for two methods, say \FisherSFT and \density, by asking a larger GPT-4o model that serves as a judge. We use the following prompt: 
\begin{quote}
\begin{small}
\begin{verbatim}
You are a judge of Shakespeare text.
<tag1>text1</tag1>
<tag2>text2</tag2>
Respond 2 if the text inside <tag2> 
is more fluent Shakespeare 
text than the text inside <tag1>. 
Respond 1 otherwise.
\end{verbatim}
\end{small}
\end{quote} 

The prompt does not name the methods, and targets our perceived benefit (improved language). We use a state-of-the-art LLM GPT-4o to judge. The text generated by the compared methods is randomized: one randomly-chosen method replaces text1 and the other text2. We tested the LLM judge and it chooses the first position with probability 0.54, which is slightly higher than 0.5 for a position-unbiased judge.
When comparing sentences generated by two approaches, we use the same initial phrase in each side-by-side comparison. One example of the outputs generated by the two models is in \cref{fig:3}. Clearly the model trained on uniformly selected sentences generates worse text, which is repetitive. In contrast, the text generated through \FisherSFT is more coherent and similar to the Shakespeare dataset. In \cref{tab:table_combined}, we report the percentage of \FisherSFT being preferred to the three different baselines \density, \askLLM, and \uniform for various samples sizes $n$. For all sample sizes, the text generated by the fine-tuned model on \FisherSFT sentences is preferred to fine-tuned models on sentences generated by the baselines. 

\section{Conclusions}
\label{sec:conclusions}
In this work we developed a method to sample training examples for a fixed budget, that greedily maximizes the log determinant of the Hessian of the log likelihood. We subsequently provide a faster version of the algorithm by leveraging sub-modularity of the problem and provide bounds on the estimation error on the model trained using the collected samples. Finally through experiments on synthetic as well as real world data we evaluate our methods and show that they perform better with lower prediction errors and better quality of sentences generated by the subsequently fine-tuned models.


\bibliographystyle{icml2025}
\bibliography{refs,brano}

\begin{thebibliography}{50}
\providecommand{\natexlab}[1]{#1}
\providecommand{\url}[1]{\texttt{#1}}
\expandafter\ifx\csname urlstyle\endcsname\relax
  \providecommand{\doi}[1]{doi: #1}\else
  \providecommand{\doi}{doi: \begingroup \urlstyle{rm}\Url}\fi

\bibitem[Abbas et~al.(2023)Abbas, Tirmala, Simig, Ganguli, and Morcos]{abbas2023}
Abbas, A., Tirmala, K., Simig, D., Ganguli, S., and Morcos, A.~S.
\newblock Semdedup: Data-efficient learning at web-scale through semantic deduplication.
\newblock \emph{arXiv preprint arXiv:2303.09540}, 2023.

\bibitem[Abbasi-Yadkori et~al.(2011)Abbasi-Yadkori, Pal, and Szepesvari]{abbasi-yadkori11improved}
Abbasi-Yadkori, Y., Pal, D., and Szepesvari, C.
\newblock Improved algorithms for linear stochastic bandits.
\newblock In \emph{Advances in Neural Information Processing Systems 24}, pp.\  2312--2320, 2011.

\bibitem[Axiotis et~al.(2024)Axiotis, Cohen-Addad, Henzinger, Jerome, Mirrokni, Saulpic, Woodruff, and Wunder]{axiotis2024data}
Axiotis, K., Cohen-Addad, V., Henzinger, M., Jerome, S., Mirrokni, V., Saulpic, D., Woodruff, D.~P., and Wunder, M.
\newblock Data-efficient learning via clustering-based sensitivity sampling: Foundation models and beyond.
\newblock In \emph{Proceedings of the 41st International Conference on Machine Learning}. PMLR, 2024.

\bibitem[Bernstein(2009)]{bernstein2009matrix}
Bernstein, D.~S.
\newblock \emph{Matrix Mathematics: Theory, Facts, and Formulas with Application to Linear Systems Theory}.
\newblock Princeton University Press, Princeton, NJ, 2nd edition, 2009.
\newblock ISBN 978-0691118028.

\bibitem[Bishop(2006)]{bishop06pattern}
Bishop, C.
\newblock \emph{Pattern Recognition and Machine Learning}.
\newblock Springer, New York, NY, 2006.

\bibitem[Bommasani et~al.(2021)]{bommasani21opportunities}
Bommasani, R. et~al.
\newblock On the opportunities and risks of foundation models.
\newblock \emph{CoRR}, abs/2108.07258, 2021.
\newblock URL \url{https://arxiv.org/abs/2108.07258}.

\bibitem[Borsos et~al.(2020)Borsos, Mutny, and Krause]{borsos2020coresets}
Borsos, Z., Mutny, M., and Krause, A.
\newblock Coresets via bilevel optimization for continual learning and streaming.
\newblock In \emph{Advances in Neural Information Processing Systems}, volume~33, pp.\  14879--14890, 2020.

\bibitem[Brown et~al.(2020)]{brown20language}
Brown, T. et~al.
\newblock Language models are few-shot learners.
\newblock In \emph{Advances in Neural Information Processing Systems 33}, 2020.

\bibitem[Chen et~al.(2012)Chen, Welling, and Smola]{chen2012super}
Chen, Y., Welling, M., and Smola, A.
\newblock Super-samples from kernel herding.
\newblock \emph{arXiv preprint arXiv:1203.3472}, 2012.

\bibitem[Chitta et~al.(2021)Chitta, {\'A}lvarez, Haussmann, and Fardet]{chitta2021training}
Chitta, K., {\'A}lvarez, J.~M., Haussmann, E., and Fardet, E.
\newblock Training data subset search with ensemble active learning.
\newblock \emph{IEEE Transactions on Intelligent Transportation Systems}, 23\penalty0 (9):\penalty0 14741--14752, 2021.

\bibitem[Coleman \& Shrivastava(2020)Coleman and Shrivastava]{ColemanShrivastava2020}
Coleman, B. and Shrivastava, A.
\newblock Sub-linear race sketches for approximate kernel density estimation on streaming data.
\newblock In \emph{Proceedings of The Web Conference 2020}, WWW '20, pp.\  1739–1749, New York, NY, USA, 2020. Association for Computing Machinery.
\newblock ISBN 9781450370233.
\newblock \doi{10.1145/3366423.3380244}.
\newblock URL \url{https://doi.org/10.1145/3366423.3380244}.

\bibitem[Coleman et~al.(2020)Coleman, Yeh, Mussmann, Mirzasoleiman, Bailis, Liang, Leskovec, and Zaharia]{coleman2020selection}
Coleman, C., Yeh, C., Mussmann, S., Mirzasoleiman, B., Bailis, P., Liang, P., Leskovec, J., and Zaharia, M.
\newblock Selection via proxy: Efficient data selection for deep learning.
\newblock In \emph{International Conference on Learning Representations}, 2020.

\bibitem[Das et~al.(2024)Das, Chakraborty, Pacchiano, and Chowdhury]{das24active}
Das, N., Chakraborty, S., Pacchiano, A., and Chowdhury, S.~R.
\newblock Active preference optimization for sample efficient {RLHF}.
\newblock \emph{CoRR}, abs/2402.10500, 2024.
\newblock URL \url{https://arxiv.org/abs/2402.10500}.

\bibitem[Feldman \& Zhang(2020)Feldman and Zhang]{feldman2020what}
Feldman, V. and Zhang, C.
\newblock What neural networks memorize and why: discovering the long tail via influence estimation.
\newblock In \emph{Advances in Neural Information Processing Systems}, volume~33, pp.\  2881--2891, 2020.

\bibitem[Fisher(1922)]{fisher22mathematical}
Fisher, R.
\newblock On the mathematical foundations of theoretical statistics.
\newblock \emph{Philosophical Transactions of the Royal Society of London: Series A}, 222:\penalty0 309--368, 1922.

\bibitem[Hajek et~al.(2014)Hajek, Oh, and Xu]{hajek}
Hajek, B., Oh, S., and Xu, J.
\newblock Minimax-optimal inference from partial rankings.
\newblock \emph{arXiv preprint arXiv:1406.5638}, 2014.
\newblock URL \url{https://arxiv.org/abs/1406.5638}.

\bibitem[Hastings(1970)]{hastings1970monte}
Hastings, W.~K.
\newblock Monte carlo sampling methods using markov chains and their applications.
\newblock \emph{Biometrika}, 57\penalty0 (1):\penalty0 97--109, 1970.

\bibitem[Hu et~al.(2022)Hu, Shen, Wallis, Allen-Zhu, Li, Wang, Wang, and Chen]{hu22lora}
Hu, E., Shen, Y., Wallis, P., Allen-Zhu, Z., Li, Y., Wang, S., Wang, L., and Chen, W.
\newblock {LoRA}: Low-rank adaptation of large language models.
\newblock In \emph{Proceedings of the 10th International Conference on Learning Representations}, 2022.

\bibitem[Indyk et~al.(2014)Indyk, Mahabadi, Mahdian, and Mirrokni]{indyk2014composable}
Indyk, P., Mahabadi, S., Mahdian, M., and Mirrokni, V.~S.
\newblock Composable core-sets for diversity and coverage maximization.
\newblock In \emph{Proceedings of the 33rd ACM SIGMOD-SIGACT-SIGART symposium on Principles of database systems}, pp.\  100--108, 2014.

\bibitem[Karnin \& Liberty(2019)Karnin and Liberty]{karnin2019discrepancy}
Karnin, Z. and Liberty, E.
\newblock Discrepancy, coreset, and sketches in machine learning.
\newblock In \emph{Conference on Learning Theory}, pp.\  1975--1993. PMLR, 2019.

\bibitem[Karpathy(2015)]{shakespeare}
Karpathy, A.
\newblock char-rnn.
\newblock \url{https://github.com/karpathy/char-rnn}, 2015.

\bibitem[Lattimore \& Szepesvari(2019)Lattimore and Szepesvari]{lattimore19bandit}
Lattimore, T. and Szepesvari, C.
\newblock \emph{Bandit Algorithms}.
\newblock Cambridge University Press, 2019.

\bibitem[Lee et~al.(2023)Lee, Miranda, and Koyejo]{lee2023beyond}
Lee, A., Miranda, B., and Koyejo, S.
\newblock Beyond scale: The diversity coefficient as a data quality metric demonstrates llms are pre-trained on formally diverse data.
\newblock \emph{arXiv preprint arXiv:2306.13840}, 2023.

\bibitem[Lee et~al.(2022)Lee, Ippolito, Nystrom, Zhang, Eck, Callison-Burch, and Carlini]{lee2022deduplicating}
Lee, K., Ippolito, D., Nystrom, A., Zhang, C., Eck, D., Callison-Burch, C., and Carlini, N.
\newblock Deduplicating training data makes language models better.
\newblock In \emph{Proceedings of the 60th Annual Meeting of the Association for Computational Linguistics (Volume 1: Long Papers)}, pp.\  8424--8445, 2022.

\bibitem[Liu et~al.(2024)Liu, Shi, and Sun]{liu24dual}
Liu, P., Shi, C., and Sun, W.~W.
\newblock Dual active learning for reinforcement learning from human feedback.
\newblock \emph{CoRR}, abs/2410.02504, 2024.
\newblock URL \url{https://arxiv.org/abs/2410.02504}.

\bibitem[Mangrulkar et~al.(2022)Mangrulkar, Gugger, Debut, Belkada, Paul, and Bossan]{peft}
Mangrulkar, S., Gugger, S., Debut, L., Belkada, Y., Paul, S., and Bossan, B.
\newblock Peft: State-of-the-art parameter-efficient fine-tuning methods.
\newblock \url{https://github.com/huggingface/peft}, 2022.

\bibitem[Meding et~al.(2021)Meding, Buschtoff, Geirhos, and Wichmann]{meding2021trivial}
Meding, R., Buschtoff, L. M.~S., Geirhos, R., and Wichmann, F.~A.
\newblock Trivial or impossible--dichotomous data difficulty makes model differences (on imagenet and beyond).
\newblock \emph{arXiv preprint arXiv:2110.05922}, 2021.

\bibitem[Mikolov et~al.(2013)Mikolov, Chen, Corrado, and Dean]{Mikolov2013EfficientEO}
Mikolov, T., Chen, K., Corrado, G.~S., and Dean, J.
\newblock Efficient estimation of word representations in vector space.
\newblock In \emph{International Conference on Learning Representations}, 2013.
\newblock URL \url{https://api.semanticscholar.org/CorpusID:5959482}.

\bibitem[Mindermann et~al.(2022)Mindermann, Brauner, Razzak, Sharma, Kirsch, Xu, H{\"o}ltgen, Gomez, Morisot, Farquhar, et~al.]{mindermann2022prioritized}
Mindermann, S., Brauner, J., Razzak, M., Sharma, M., Kirsch, A., Xu, W., H{\"o}ltgen, B., Gomez, A., Morisot, A., Farquhar, S., et~al.
\newblock Prioritized training on points that are learnable, worth learning, and not yet learnt.
\newblock In \emph{International Conference on Machine Learning}, pp.\  15630--15649. PMLR, 2022.

\bibitem[Muenchigoff et~al.(2023)Muenchigoff, Rush, Barak, Scao, Piktus, Tazi, Pyysalo, Wolf, and Raffel]{muenchigoff2023scaling}
Muenchigoff, M., Rush, A.~M., Barak, B., Scao, T.~L., Piktus, T., Tazi, N., Pyysalo, S., Wolf, T., and Raffel, C.
\newblock Scaling data-constrained language models.
\newblock \emph{arXiv preprint arXiv:2305.10623}, 2023.

\bibitem[Mukherjee et~al.(2024)Mukherjee, Lalitha, Kalantari, Deshmukh, Liu, Ma, and Kveton]{mukherjee24optimal}
Mukherjee, S., Lalitha, A., Kalantari, K., Deshmukh, A., Liu, G., Ma, Y., and Kveton, B.
\newblock Optimal design for human preference elicitation.
\newblock In \emph{Advances in Neural Information Processing Systems 37}, 2024.

\bibitem[Nemhauser et~al.(1978)Nemhauser, Wolsey, and Fisher]{nemhauser78approximation}
Nemhauser, G.~L., Wolsey, L.~A., and Fisher, M.~L.
\newblock An analysis of approximations for maximizing submodular set functions - {I}.
\newblock \emph{Mathematical Programming}, 14\penalty0 (1):\penalty0 265--294, 1978.

\bibitem[Ouyang et~al.(2022)Ouyang, Wu, Jiang, Almeida, Wainwright, Mishkin, Zhang, Agarwal, Slama, Ray, Schulman, Hilton, Kelton, Miller, Simens, Askell, Welinder, Christiano, Leike, and Lowe]{ouyang22training}
Ouyang, L., Wu, J., Jiang, X., Almeida, D., Wainwright, C., Mishkin, P., Zhang, C., Agarwal, S., Slama, K., Ray, A., Schulman, J., Hilton, J., Kelton, F., Miller, L., Simens, M., Askell, A., Welinder, P., Christiano, P., Leike, J., and Lowe, R.
\newblock Training language models to follow instructions with human feedback.
\newblock In \emph{Advances in Neural Information Processing Systems 35}, 2022.

\bibitem[Paul et~al.(2021)Paul, Ganguli, and Dziugaite]{paul2021deep}
Paul, M., Ganguli, S., and Dziugaite, G.~K.
\newblock Deep learning on a data diet: Finding important examples early in training.
\newblock In \emph{Advances in Neural Information Processing Systems}, volume~34, pp.\  2960--2971, 2021.

\bibitem[Phillips(2017)]{phillips2017coresets}
Phillips, J.~M.
\newblock Coresets and sketches.
\newblock In \emph{Handbook of discrete and computational geometry}, pp.\  1269--1288. Chapman and Hall/CRC, 2017.

\bibitem[Pukelsheim(2006)]{Pukelsheim1993OptimalDO}
Pukelsheim, F.
\newblock \emph{Optimal Design of Experiments}, volume~50 of \emph{Classics in Applied Mathematics}.
\newblock Society for Industrial and Applied Mathematics, Philadelphia, PA, 2006.
\newblock ISBN 0898716047.

\bibitem[Radford et~al.(2019)Radford, Wu, Child, Luan, Amodei, and Sutskever]{radford2019language}
Radford, A., Wu, J., Child, R., Luan, D., Amodei, D., and Sutskever, I.
\newblock Language models are unsupervised multitask learners.
\newblock OpenAI Technical Report, 2019.
\newblock \url{https://cdn.openai.com/better-language-models/language_models_are_unsupervised_multitask_learners.pdf}.

\bibitem[Rafailov et~al.(2023)Rafailov, Sharma, Mitchell, Manning, Ermon, and Finn]{rafailov23direct}
Rafailov, R., Sharma, A., Mitchell, E., Manning, C., Ermon, S., and Finn, C.
\newblock Direct preference optimization: Your language model is secretly a reward model.
\newblock In \emph{Advances in Neural Information Processing Systems 36}, 2023.

\bibitem[Sachdeva et~al.(2021)Sachdeva, Wu, and McAuley]{sachdeva2021svp}
Sachdeva, N., Wu, C.-J., and McAuley, J.
\newblock {SVP}-{CF}: Selection via proxy for collaborative filtering data.
\newblock \emph{arXiv preprint arXiv:2107.04984}, 2021.

\bibitem[Sachdeva et~al.(2024)Sachdeva, Coleman, Kang, Ni, Hong, Chi, Caverlee, and Cheng]{sachdeva2024how}
Sachdeva, N., Coleman, B., Kang, W.-C., Ni, J., Hong, L., Chi, E.~H., Caverlee, J., and Cheng, D.~Z.
\newblock How to train data-efficient llms.
\newblock \emph{arXiv preprint arXiv:2402.09668}, 2024.

\bibitem[Scheid et~al.(2024)Scheid, Boursier, Durmus, Jordan, Menard, Moulines, and Valko]{scheid24optimal}
Scheid, A., Boursier, E., Durmus, A., Jordan, M., Menard, P., Moulines, E., and Valko, M.
\newblock Optimal design for reward modeling in {RLHF}.
\newblock \emph{CoRR}, abs/2410.17055, 2024.
\newblock URL \url{https://arxiv.org/abs/2410.17055}.

\bibitem[Sorscher et~al.(2022)Sorscher, Geirhos, Shekhar, Ganguli, and Morcos]{sorscher2022beyond}
Sorscher, B., Geirhos, R., Shekhar, S., Ganguli, S., and Morcos, A.~S.
\newblock Beyond neural scaling laws: beating power law scaling via data pruning.
\newblock In \emph{Advances in Neural Information Processing Systems}, volume~35, pp.\  19523--19536, 2022.

\bibitem[Stufken \& Yang(2012)Stufken and Yang]{stufken12optimal}
Stufken, J. and Yang, M.
\newblock Optimal designs for generalized linear models.
\newblock In \emph{Design and Analysis of Experiments}, pp.\  137--164. John Wiley \& Sons, 2012.

\bibitem[Thekumparampil et~al.(2024)Thekumparampil, Hiranandani, Kalantari, Sabach, and Kveton]{thekumparampil24comparing}
Thekumparampil, K., Hiranandani, G., Kalantari, K., Sabach, S., and Kveton, B.
\newblock Comparing few to rank many: Active human preference learning using randomized {Frank-Wolfe}.
\newblock \emph{CoRR}, abs/2412.19396, 2024.
\newblock URL \url{https://arxiv.org/abs/2412.19396}.

\bibitem[Tirmala et~al.(2023)Tirmala, Simig, Aghajanyan, and Morcos]{tirmala2023d4}
Tirmala, K., Simig, D., Aghajanyan, A., and Morcos, A.~S.
\newblock {D4}: Improving lm pre-training via document de-duplication and diversification.
\newblock \emph{arXiv preprint arXiv:2308.12284}, 2023.

\bibitem[Tukan et~al.(2021)Tukan, Baykal, Feldman, and Rus]{tukan2021coresets}
Tukan, M., Baykal, C., Feldman, D., and Rus, D.
\newblock On coresets for support vector machines.
\newblock \emph{Theoretical Computer Science}, 890:\penalty0 171--191, 2021.

\bibitem[Wei et~al.(2022)Wei, Bosma, Zhao, Guu, Yu, Lester, Du, Dai, and Le]{wei22finetuned}
Wei, J., Bosma, M., Zhao, V., Guu, K., Yu, A.~W., Lester, B., Du, N., Dai, A., and Le, Q.
\newblock Finetuned language models are zero-shot learners.
\newblock In \emph{Proceedings of the 10th International Conference on Learning Representations}, 2022.

\bibitem[Wenzek et~al.(2019)Wenzek, Lachaux, Conneau, Chaudhary, Guzm{\'a}n, Joulin, and Grave]{wenzek2019ccnet}
Wenzek, G., Lachaux, M.-A., Conneau, A., Chaudhary, V., Guzm{\'a}n, F., Joulin, A., and Grave, E.
\newblock Ccnet: Extracting high quality monolingual datasets from web crawl data.
\newblock \emph{arXiv preprint arXiv:1911.00359}, 2019.

\bibitem[Wolf et~al.(2020)Wolf, Debut, Sanh, Chaumond, Delangue, Moi, Cistac, Rault, Louf, Funtowicz, Davison, Shleifer, von Platen, Ma, Jernite, Plu, Xu, Scao, Gugger, Drame, Lhoest, and Rush]{wolf2020huggingfacestransformersstateoftheartnatural}
Wolf, T., Debut, L., Sanh, V., Chaumond, J., Delangue, C., Moi, A., Cistac, P., Rault, T., Louf, R., Funtowicz, M., Davison, J., Shleifer, S., von Platen, P., Ma, C., Jernite, Y., Plu, J., Xu, C., Scao, T.~L., Gugger, S., Drame, M., Lhoest, Q., and Rush, A.~M.
\newblock Huggingface's transformers: State-of-the-art natural language processing, 2020.
\newblock URL \url{https://arxiv.org/abs/1910.03771}.

\bibitem[Zhu et~al.(2023)Zhu, Jiao, and Jordan]{zhu23principled}
Zhu, B., Jiao, J., and Jordan, M.
\newblock Principled reinforcement learning with human feedback from pairwise or {$K$}-wise comparisons.
\newblock \emph{CoRR}, abs/2301.11270, 2023.
\newblock URL \url{https://arxiv.org/abs/2301.11270}.

\end{thebibliography}

\newpage
\appendix
\onecolumn

\onecolumn

\section{Related Works}
\label{sec:related}

\begin{algorithm}[h]
\caption{Inverse Propensity Sampling (IPS) via Kernel Density Estimation (KDE) \cite{sachdeva2021svp}}
\label{alg:IPS}
\begin{algorithmic}[1]
\State Dataset $\mathcal{D} = \{x_1, x_2, \ldots, x_N\}$ of embeddings, 
         sample size $k$, kernel $k$ with corresponding LSH family $\mathcal{H}$ 
         \cite{ColemanShrivastava2020}, hash range $B$, rows $R$, random seed $s$.
\State Ensure a subset of $\mathcal{D}$ of size $k$, sampled with probability $p$ (see line~\ref{line:final}).

\State Initialize KDE sketch $S \leftarrow \mathbf{0}^{R \times B}$.
\State Generate $R$ independent hash functions $h_1, \dots, h_R$ from $\mathcal{H}$ with range $B$ and random seed $s$.

\For{$n \gets 1$ to $N$}
  \For{$r \gets 1$ to $R$}
    \State $S_{r,\, h_r(x_n)} \leftarrow S_{r,\, h_r(x_n)} + 1$
  \EndFor
\EndFor

\State Initialize a list of scores $\mathcal{S} \gets [\,]$.

\For{$n \gets 1$ to $N$}
  \State $score \gets 0$
  \For{$r \gets 1$ to $R$}
    \State $score \gets score + S[r,\, h_r(x_n)]$
  \EndFor
  \State Append $\frac{score}{R}$ to $\mathcal{S}$.
\EndFor

\State \label{line:final}\textbf{Output:} Select $k$ elements from $\mathcal{D}$ with probability 
       $p = \frac{S}{\sum S}$ (sampled without replacement).

\end{algorithmic}
\end{algorithm}

Coverage-oriented approaches center on ensuring that a training set reflects the entire input distribution as broadly as possible. One common strategy is \emph{cluster sampling} \cite{lee2023beyond}, which embeds data points in a metric space (often via learned representations) and selects mutually distant examples to form “coresets” \cite{phillips2017coresets,tukan2021coresets}. Related methods include \emph{prototype-based sampling} for vision \cite{sorscher2022beyond} and \emph{deduplication algorithms} \cite{abbas2023,lee2022deduplicating,tirmala2023d4} that remove near-duplicates or redundancies. More sophisticated procedures—such as \emph{submodular optimization} \cite{chen2012super,indyk2014composable,borsos2020coresets} and \emph{discrepancy minimization} \cite{karnin2019discrepancy}—further refine coverage by balancing representation across diverse data regions.

Quality-based sampling, in contrast, prioritizes weeding out low-value or unhelpful examples. A prominent technique is \emph{perplexity filtering} \cite{wenzek2019ccnet,muenchigoff2023scaling}, which prefers samples with higher likelihood under a pretrained model, though this can inadvertently discard valuable but rare text. Other approaches compute “uncertainty scores” via ensemble disagreement \cite{chitta2021training,meding2021trivial} or examine whether examples are \emph{memorized} \cite{feldman2020what} or \emph{unlearnable} \cite{mindermann2022prioritized}. The \emph{SVP algorithm} \cite{coleman2020selection,sachdeva2021svp} estimates each sample’s importance by its validation-loss variance, while \emph{EL2N scores} \cite{paul2021deep} track a model’s difficulty in learning particular data points. These methods all fit into a “score-and-sample” framework \cite{hastings1970monte}, where the final selection depends on the magnitude of each item’s quality score.

For a more detailed description see \cite{sachdeva2024how}. Below we describe the two algorithms proposed in \cite{sachdeva2024how} and used as benchmarks in Section~\ref{sec:experiments}.

\textbf{ASK-LLM:} In ASK-LLM \cite{sachdeva2024how}, a proxy LLM is prompted with a potential training example and asked whether the example should be used for training. More specifically, the proxy LLM is provided the training example followed by the prompt "Does the previous paragraph contain informative signal for fine-tuning a large-language model?
An informative datapoint should be well-formatted, contain some usable knowledge of the world, and strictly NOT have any harmful, racist, sexist, etc. content. OPTIONS: yes, no".
It then takes the softmax probability of the token “yes” as the estimated data-quality score and sorts according to score to pick Top n data points.

\textbf{Density sampling:} \cite{sachdeva2024how} 
    assumes access to embeddings from a pre-trained LLM. Given a dataset \( D \) it uses a kernel \( k(x, y) \), to estimate the density using the following score.
    \[
    \text{score}(y) = \sum_{x \in D} k_{\lambda}(x, y),
    \]
    where \(\lambda\) is a smoothing parameter and controls the scale of the data points’ effects. Density Sampling then uses Inverse propensity sampling (IPS) to select items proportional to their re-weighted and normalized inverse score. The algorithm as provided in \cite{sachdeva2024how} is summarized below.

\textbf{Clustering Based Sensitivity Sampling:} \cite{axiotis2024data} The method uses k-means clustering and sensitivity sampling using the embedding representation of the data with respect to which the model loss is measured and ensures that the sampled elements' average loss corresponds to the average loss of the whole dataset. The algorithm as presented in \cite{axiotis2024data} is summarized below.
\begin{algorithm}[h]
\caption{Clustering Based Sensitivity Sampling (\( \mathcal{D}, k, \varepsilon, \Lambda, C \)) \cite{axiotis2024data}}
\begin{algorithmic}[1]
\State \textbf{Input:} a dataset \( \mathcal{D} \) partitioned into clusters \( C = (C_1, \dots, C_k) \) with centers \( c_1, \dots, c_k \) and a \( k \)-tuple of parameters \( \Lambda_1, \dots, \Lambda_k \).
\For{ \( e \in C_i \) }
    \State Define \( \hat{\ell}(e) := \ell(c_i) \) and \( v(e) := \| e - c_i \|^z \).
\EndFor
\State Let \( s := \lceil \varepsilon^{-2}(2 + 2\varepsilon / 3) \rceil \). For \( e \in C_i \) define \( p_e := \frac{\hat{\ell}(e) + \Lambda_i v(e)}{\sum_i \Lambda_i \Phi(C_i, \{c_i\}) + \sum_{x \in \mathcal{D}} \hat{\ell}(x)} \) and \( w(e) = s^{-1} p_e^{-1} \).
\State Compute a sample \( S \) of \( s \) points, picked independently following the distribution \( p_e \).
\State \textbf{Output:} the set \( S \) with weights \( w \).
\end{algorithmic}
\end{algorithm}

\section{Gradient and Hessian of the Loss}
\label{sec:app_grad}
\label{sec:grad_Hessian}
\begin{proposition}
    Consider the Loss function as defined in \eqref{eq:loglik_subset} and suppose assumption \ref{asmp:feature} holds. Then the gradient and Hessian of $\cL_{\cS}$ are respectively given by 
    \begin{align*}
         \nabla \cL_{\cS}(\Theta) = \frac{1}{n}\sum_{i \in \cS} \sum_{j \in [M_i]} \vec\Big(\x_{i,j} \otimes \big(p(y_{i,j} = \ell| \x_{i,j};\Theta) - \1(y_{i,j})\big)\Big) \\
        \nabla^2 \cL_{\cS}(\Theta) = \frac{1}{n} \sum_{i \in \cS} \sum_{j \in [M_i]} \Big(\mathbf{diag}(p(y_{i,j}| \x_{i,j};\Theta)) - p(y_{i,j}| \x_{i,j};\Theta)p(y_{i,j}| \x_{i,j};\Theta)^\top \Big) \otimes \x_{i,j}\x_{i,j}^\top
    \end{align*}
    \label{prop:grad_hessian}
\end{proposition}

\begin{proof}
    Recall that the loss function is given by
    \begin{align*}
        \cL_{\cS}(\Theta) &= - \frac{1}{n} \sum_{i \in \cS} \sum_{j \in [M_i]} \sum_{\ell \in [L]} \log P(y_{i,j} = \ell | \x_{i,j}, \Theta) \delta(y_{i,j} = \ell)\\
        &= - \frac{1}{n} \sum_{i \in \cS} \sum_{j \in [M_i]} \sum_{\ell \in [L]} \log \left(\frac{\exp\big((\Theta^T\x_{i,j})_{\ell}\big)}{\overset{L}{\underset{\ell' = 1}{\sum}}\exp\big((\Theta^T\x_{i,j})_{\ell'}\big)}\right) \delta(y_{i,j} = \ell).
    \end{align*}


Now the loss can be re-written as 
\begin{align*}
\cL_{\cS}(\Theta) = \frac{-1}{n}\sum_{i \in \cS} \sum_{j \in [M_i]}  \left[ \theta_{y_{i,j}}^T \x_{i,j} - \log \sum_{\ell=1}^{L} \exp (\theta_{\ell}^T \x_n) \right]
\end{align*}
Now note that 
\begin{align*}
\frac{\partial}{\partial \theta_{\ell}} \theta_{y_{i,j}}^T \x_{i,j} = \delta(y_{i,j} = \ell) \x_{i,j}
\end{align*}
and that,
\begin{align*}
\frac{\partial}{\partial \theta_{\ell}} \log \sum_{{\ell'}=1}^{L} \exp (\theta_{{\ell}'}^T \x_{i,j}) &= \frac{\sum_{\ell'=1}^{L} \exp (\theta_{\ell'}^T \x_{i,j}) \times \delta(y_{i,j} = \ell) \x_{i,j}}{\sum_{k=1}^{L} \exp (\theta_{k}^T \x_{i,j})}\\
&= \sum_{{\ell'}=1}^{L} p(y_{i,j} = \ell'| \x_{i,j};\Theta) \delta(y_{i,j} = \ell) \x_{i,j}\\
&= p(y_{i,j} = \ell| \x_{i,j};\Theta) \x_{i,j}
\end{align*}
Combining both we get
\begin{align*}
    \frac{\partial}{\partial \theta_{\ell}}\cL_{\cS}(\Theta) = \frac{-1}{n}\sum_{i \in \cS} \sum_{j \in [M_i]}\Big(\delta(y_{i,j} = \ell) - p(y_{i,j} = \ell| \x_{i,j};\Theta)\Big) \x_{i,j}
\end{align*}
Therefore the gradient of the loss $\cL_{\cS}(\Theta)$ with respect to ${\Theta}$ is given by
\begin{align}
\label{eq:gradient_equation}
    \nabla \cL_{\cS}(\Theta) = \frac{-1}{n}\sum_{i \in \cS} \sum_{j \in [M_i]} \vec\Big(\x_{i,j} \otimes \big(\1(y_{i,j}) - p(y_{i,j} = \ell| \x_{i,j};\Theta)\big)  \Big)
\end{align}
where $\1(y_{i,j}) \in \R^L$ is a one-hot vector with the $y_{i,j}$-th entry as 1 and $\otimes$ is the Kronecker product.

Next we compute the Hessian. Note that 
\begin{align*}
    \frac{\partial^2}{\partial \theta_{\ell} \theta_{\ell'}}\cL_{\cS}(\Theta) &= \frac{-1}{n} \frac{\partial}{\partial \theta_{\ell}} \sum_{i \in \cS} \sum_{j \in [M_i]} \Big(\delta(y_{i,j} = \ell) - p(y_{i,j} = \ell| \x_{i,j};\Theta)\Big) \x_{i,j}\\
    &= \frac{1}{n} \sum_{i \in \cS} \sum_{j \in [M_i]} \Bigg(\frac{\partial}{\partial \theta_{\ell'}} p(y_{i,j} = \ell| \x_{i,j};\Theta)\Bigg) \x_{i,j}^T \\
    &= \frac{1}{n} \sum_{i \in \cS} \sum_{j \in [M_i]} p(y_{i,j} = \ell| \x_{i,j};\Theta)\Big(\delta(\ell = \ell') - p(y_{i,j} = \ell| \x_{i,j};\Theta)\Big)\x_{i,j}\x_{i,j}^\top
\end{align*}
and therefore, the Hessian of the loss is given by
\begin{align}
    \nabla^2 \cL_{\cS}(\Theta) = \frac{1}{n} \sum_{i \in \cS} \sum_{j \in [M_i]} \Big(\mathbf{diag}(p(y_{i,j}| \x_{i,j};\Theta)) - p(y_{i,j}| \x_{i,j};\Theta)p(y_{i,j}| \x_{i,j};\Theta)^\top \Big) \otimes \x_{i,j}\x_{i,j}^\top
    \label{eq:hessian_app}
\end{align}
Now note that $p(y_{i,j}| \x_{i,j};\Theta) \geq e^{-2\alpha}$ where $\sup_{\ell,i,j} \big|\Theta_{\ell}^\top \x_{i,j}\big| \leq \alpha$.

Therefore 
\begin{align}
    \nabla^2 \cL_{\cS}(\Theta) = \frac{1}{n} \sum_{i \in \cS} \sum_{j \in [M_i]} \Big(e^{-2\alpha} \mathbf{I}_{L\times L} - e^{-4\alpha} \ones\ones^{\top} \Big) \otimes \x_{i,j}\x_{i,j}^\top
    \label{eq:hessian_lower}
\end{align}
Assume $\Big(e^{-2\alpha} \mathbf{I}_{L\times L} - e^{-4\alpha} \ones\ones^{\top} \Big) \succeq \gamma \mathbf{I_{l\times L}}$ for some $\gamma > 0$. Then we have
\begin{align}
    \nabla^2 \cL_{\cS}(\Theta) \succeq \frac{1}{n} \sum_{i \in \cS} \sum_{j \in [M_i]} \gamma \mathbf{I_{L \times L}} \otimes \x_{i,j}\x_{i,j}^\top
    \label{eq:hessian_lower_gamma}
\end{align}
\end{proof}

\section{Proof of Error Bound}
\label{sec:app_proof}
\curvatureProp*
\begin{proof}

We derive an upper bound on $\normw{\x_{i,j}}{\bar{\Sigma}_n^{-1}}$, where $\x_{i,j} \in \realset^d$ is a feature vector and $\bar{\Sigma}_n \in \realset^{d \times d}$ is a design matrix obtained by greedy log-determinant maximization. Let $\cD = \set{\x_{i,j}: i\in [N], j \in [M_i]}$ be a dataset of $N$ data points such that $\normw{\x_{i,j}}{2} \leq 1$. Let $I_t \in [N]$ be the index of the $t$-th chosen feature vector and $\cS_t = \set{I_\ell}_{\ell = 1}^t$ be the first $t$ chosen feature vectors. For simplicity we use $\bar{\Sigma}_{\cS_n}$ and $\bar{\Sigma}_{n}$ interchangeably. Let
\begin{align*}
    \bar{\Sigma}_{t} = \sigma_0^2 \bm{I} + \sum_{i \in \cS_t} \sum_{j = 1}^{M_i} \x_{i,j}\x_{i,j}^\top
\end{align*}
where $\sigma_0 > 0$ is a constant that guarantees that $\Sigma_0$ is well defined.

The $t$-th feature vector is chosen as
\begin{align}
  I_t
  = \argmax_{i \in [N] \setminus \cS_{t - 1}}
  \log\det \left(\bar{\Sigma}_{t - 1} + \sum_{j=1}^{M_t}\x_{t,j}\x_{t,j}\T\right)\,.
  \label{eq:greedy logdet maximization}
\end{align}

\begin{lemma}
\label{lem:variance monotonicity} For any $i \in [N]$ and $t \in [n]$,
\begin{align*}
  \sum_{j=1}^{M_i}\x_{i,j}\T\bar{\Sigma}_t^{-1} \x_{i,j}
  \leq \sum_{j=1}^{M_i}\x_{i,j}\T\bar{\Sigma}_{t-1}^{-1} \x_{i,j}.
\end{align*}
\end{lemma}
\begin{proof}
Define the matrix
\[
X \;=\; \begin{bmatrix}
\x_{i,1} & \x_{i,2} & \cdots & \x_{i,M_t}
\end{bmatrix},
\]
so that each $\x_{i,j}$ is a column of $X$. Then we can write
\[
\sum_{i=1}^{M_t} \x_{i,j}\,\x_{i,j}^T \;=\; X\,X^T.
\]
Hence we want to find the inverse of
\[
\Sigma_{t-1} + X\,X^T.
\]

Using \textbf{Sherman--Morrison--Woodbury identity}, which states that for an invertible matrix $A$ and any matrices $U, C, V$ of compatible dimensions (with $C$ also invertible), one has
\[
\bigl(A + U\,C\,V\bigr)^{-1}
\;=\;
A^{-1}
\;-\;
A^{-1} \,U\,\bigl(C^{-1} + V\,A^{-1}\,U\bigr)^{-1}\,V\,A^{-1}.
\]

In our case, we set
\[
A = \Sigma_{t-1}, \quad U = X, \quad C = I_{M_t}, \quad V = X^T,
\]
where $I_{M_t}$ is the $M_t\times M_t$ identity matrix. Then
\[
A + U\,C\,V \;=\; \Sigma_{t-1} + X\,I_n\,X^T \;=\; \Sigma_{t-1} + X\,X^T.
\]
By applying the identity, we get
\[
(\Sigma_{t-1} + X\,X^T)^{-1}
\;=\;
\Sigma_{t-1}^{-1}
\;-\;
\Sigma_{t-1}^{-1} \,X
\,\bigl(I_{M_t} + X^T\,\Sigma_{t-1}^{-1}\,X\bigr)^{-1}
\,X^T\,\Sigma_{t-1}^{-1}.
\]
which implies
\begin{align*}
  \bar{\Sigma}_t^{-1}
  \preceq \bar{\Sigma}_{t - 1}^{-1}\,,
\end{align*}
we get $v\T \Sigma_t^{-1} v \leq v\T \Sigma_{t - 1}^{-1} v$ for any vector $v \in \realset^d$. This concludes the proof.
\end{proof}

\noindent \cref{lem:variance monotonicity} implies that
\begin{align*}
  \sum_{j=1}^{M_i}\x_{i,j}\T\bar{\Sigma}_n^{-1} \x_{i,j}
  \leq \frac{1}{n} \sum_{t=1}^{n} \sum_{j=1}^{M_i}\x_{i,j}\T\bar{\Sigma}_{t}^{-1} \x_{i,j}.
\end{align*}
holds for any $i \in [N]$. This allows us to attribute the quality of the solution to individual greedy steps in \eqref{eq:greedy logdet maximization}. 


If the scope of the maximization was $i \in [N]$, the inequality $\sum_{j=1}^{M_i} \x_{i,j}\T \bar{\Sigma}_{t - 1}^{-1} \x_{i,j} \leq \sum_{j=1}^{M_t} \x_{I_t,j}\T \bar{\Sigma}_{t - 1}^{-1} \x_{I_t,j}$ would hold for any $i \in [N]$. Since the scope is $i \in [N] \setminus \cS_{t - 1}$, we make \cref{ass:diverse-feature-vectors}.


\noindent We also use the following logarithmic transformation.

\begin{lemma}
\label{lem:logarithmic transformation} For any $i \in [N]$ and $t \in [n]$,
\begin{align*}
  \sum_{j=1}^{M_i} \x_{i,j}\T \Sigma_{t - 1}^{-1} \x_{i,j}
  \leq \frac{\sigma_0^{-2} \log\left(1 + \frac{\sigma_0^{-2} n M}{d}\right)}
  {\log(1 + \sigma_0^{-2})} \frac{\kappa d}{n}\,.
\end{align*}
\end{lemma}
\begin{proof}
We start with an upper bound on $\sum_{j=1}^{M_i} \x_{i,j}\T \Sigma_{t - 1}^{-1} \x_{i,j}$. By Weyl's inequalities, we have
\begin{align*}
  \lambda_1(\Sigma_{t - 1}^{-1})
  = \lambda_d^{-1}(\Sigma_{t - 1})
  \leq \lambda_d^{-1}(\sigma_0^{2} I_d)
  = \sigma_0^{-2}\,.
\end{align*}
Therefore, under the assumption that $\normw{\x_{i,j}}{2} \leq 1$, we have $\sum_{j=1}^{M_i} \x_{i,j}\T \Sigma_{t - 1}^{-1} \x_{i,j} \leq \sigma_0^{-2} M_i$. Now note that for any $x \in [0, u]$,
\begin{align*}
  x
  = \frac{x}{\log(1 + x)} \log(1 + x)
  \leq \left(\max_{x \in [0, u]} \frac{x}{\log(1 + x)}\right) \log(1 + x)
  = \frac{u}{\log(1 + u)} \log(1 + x)\,.
\end{align*}
Finally, we set $x = \sum_{j=1}^{M_i} \x_{i,j}\T \Sigma_{t - 1}^{-1} \x_{i,j}$ and $u = \sigma_0^{-2} M_i$, and get our claim.
\end{proof}

\begin{assumption}
\label{ass:diverse feature vectors} There exists a constant $\kappa \geq 1$ such that
\begin{align*}
\log\det(I_d + \sum_{j=1}^{M_i} \Sigma_{t-1}^{-1/2} x_{i,j}x_{i,j}^\top \Sigma_{t-1}^{-1/2})
& \leq \kappa \log\det(I_d + \sum_{j=1}^{M_{I_t}} \Sigma_{t-1}^{-1/2} x_{I_t,j}x_{I_t,j}^\top \Sigma_{t-1}^{-1/2})
\end{align*}
holds for any $i \in \mathcal{S}_{t-1}$ and $t \in [n]$.
\end{assumption}

\noindent Now we apply \cref{ass:diverse feature vectors} and \cref{lem:logarithmic transformation}, use the telescoping property of the sum, and $M = \max_{i \in [N]} M_i$ to get

\begin{align*}
  \sum_{t = 1}^n \sum_{j=1}^{M_i} \x_{i,j}\T \Sigma_{t - 1}^{-1} \x_{i,j}
  &\leq \sum_{t=1}^n \sum_{j=1}^{M_i} \frac{\sigma_0^{-2}}{\log(1+\sigma_0^{-2})} \log(1 + x_{i,j}^\top \Sigma_{t-1}^{-1}  x_{i,j}) \\
  &\leq \frac{\sigma_0^{-2}}{\log(1+\sigma_0^{-2})} \sum_{t=1}^n \sum_{j=1}^{M_i}  \log\det(I_{d} + \Sigma_{t-1}^{-1/2} x_{i,j}x_{i,j}^\top \Sigma_{t-1}^{-1/2}) \\
  &\leq \frac{\sigma_0^{-2} M_i}{\log(1+\sigma_0^{-2})}  \sum_{t=1}^n  \log\det(I_d + \frac1{M_i} \sum_{j=1}^{M_i} \Sigma_{t-1}^{-1/2} x_{i,j}x_{i,j}^\top \Sigma_{t-1}^{-1/2}) \\
  & \leq \frac{\sigma_0^{-2} M}{\log(1+\sigma_0^{-2})}  \sum_{t=1}^n  \log\det(I_d + \sum_{j=1}^{M_i} \Sigma_{t-1}^{-1/2} x_{i,j}x_{i,j}^\top \Sigma_{t-1}^{-1/2}) \\
  & \leq \frac{\sigma_0^{-2} M}{\log(1+\sigma_0^{-2})}  \sum_{t=1}^n  \kappa \log\det(I_d + \sum_{j=1}^{M_{I_t}} \Sigma_{t-1}^{-1/2} x_{I_t,j}x_{I_t,j}^\top \Sigma_{t-1}^{-1/2}) \\
  & = \frac{\kappa \sigma_0^{-2} M}{\log(1+\sigma_0^{-2})} \sum_{t=1}^n  \log\det(\Sigma_{t-1} + \sum_{j=1}^{M_{I_t}} x_{I_t,j}x_{I_t,j}^\top) - \log\det(\Sigma_{t-1}) \\
  & = \frac{\kappa \sigma_0^{-2} M}{\log(1+\sigma_0^{-2})} \sum_{t=1}^n \log\det(\Sigma_{t}) - \log\det(\Sigma_{t-1}) \\
  & = \frac{\kappa \sigma_0^{-2} M}{\log(1+\sigma_0^{-2})} (\log\det(\Sigma_{n}) - \log\det(\Sigma_{0})) \\
  & = \frac{\kappa \sigma_0^{-2} M}{\log(1+\sigma_0^{-2})} (\log\det(\Sigma_{n}) - d \log(\sigma_{0}^2))
\end{align*}

Furthermore,
\begin{align*}
  \log\det(\Sigma_n )
  & \leq d \log\left(\frac{1}{d}
  \trace\left(\Sigma_n\right)\right)
  = d \log\left(1 + \frac{1}{d} \sum_{t = 1}^n
  \trace\left(\sum_{j=1}^{M_{I_t}} \x_{I_t,j} \x_{I_t,j}^\top
  \right)\right) \\
  & = d \log\left(\sigma_0^2 I_d + \frac{1}{d} \sum_{t = 1}^n
  \sum_{j=1}^{M_{I_t}} \x_{I_t,j}\T x_{I_t,j}\right)
  \leq d \log\left( \sigma_0^{2} + \frac{n M}{d}\right)\,.
\end{align*}
Finally, we combine all claims and get
\begin{align*}
  \max_{i \in [N]} \sum_{j=1}^{M_i} \x_{i,j}\T \Sigma_{n}^{-1} \x_{i,j}
  \leq \frac{\kappa}{n} \frac{\sigma_0^{-2} M}{\log(1+\sigma_0^{-2})} (d \log\det(\frac1d \mathrm{tr}(\sum_{t=1}^n \sum_{j=1}^{M_{I_t}} x_{i,j} x_{i,j}^\top)) - d\log(\sigma_{0}))
  \leq \frac{\sigma_0^{-2} \log\left(1 + \frac{\sigma_0^{-2} n M}{d}\right)}
  {\log(1 + \sigma_0^{-2})} \frac{\kappa d}{n}\,.
\end{align*}
This concludes the proof.

\end{proof}
\stronglyConvexLemma*
\begin{proof}
    Using Taylor's expansion
\begin{align*}
    \cL_{\cS}(\Theta^*) + \langle \nabla \cL_{\cS}(\Theta_*), \hat{\Theta} - \Theta^* \rangle + \langle \hat{\Theta} - \Theta^*, \nabla^2 \cL_{\cS}(\Theta), \hat{\Theta} - \Theta^* \rangle = \cL_{\cS}(\hat{\Theta})
\end{align*}
The Hessian is given by
\begin{align*}
    \nabla^2 \cL_{\cS}(\Theta) &= \frac{1}{n} \sum_{i \in \cS} \sum_{j=1}^{M_i} (\text{diag}(\mathbf{p}_{i,j}) - \mathbf{p}_{i,j}\mathbf{p}_{i,j}^\top) \otimes (\x_{i,j}\x_{i,j}^\top)
\end{align*}
where $\mathbf{p}_{i,j} = p(y_{i,j}| \x_{i,j};\Theta)$. Now using Claim 1 from \cite{hajek} we have
\begin{align*}
    e^{2\alpha} (\text{diag}(\mathbf{p}_{i,j}) - \mathbf{p}_{i,j}\mathbf{p}_{i,j}^\top) \succeq \frac{1}{L} \mathbf{I}_L + \frac{1}{L^2} \1\1^{\top}
\end{align*}
where $\alpha = \max_{i,j} |\theta_{*,y_{i,j}}^\top \x_{i,j}| \leq 1$. Therefore we have
\begin{align*}
    \nabla^2 \cL_{\cS}(\Theta) &= \frac{1}{n} \sum_{i \in \cS} \sum_{j=1}^{M_i} (\text{diag}(\mathbf{p}_{i,j}) - \mathbf{p}_{i,j}\mathbf{p}_{i,j}^\top) \otimes (\x_{i,j}\x_{i,j}^\top)\\
    & \succeq \frac{1}{n} \sum_{i \in \cS} \sum_{j=1}^{M_i} \left( \frac{e^{-2\alpha}}{L} \bI_{L\times L} - \frac{e^{-2\alpha}}{L^2} \1\1^{\top} \right) \otimes (\x_{i,j}\x_{i,j}^\top)
\end{align*}
Now consider $\langle \hat{\Theta} - \Theta^*, \nabla^2 \cL_{\cS}(\Theta), \hat{\Theta} - \Theta^* \rangle$. We can express this as follows:
\begin{align*}
    \langle \hat{\Theta} - \Theta^*, \nabla^2 \cL_{\cS}(\Theta), \hat{\Theta} - \Theta^* \rangle & = \frac{1}{n} \sum_{i \in \cS} \sum_{j=1}^{M_i} \sum_{k,k'} \left(\sqrt{\text{diag}(p_{i,j})}\Delta \Theta_{\cdot,k}\right)^\top\left(\sqrt{\text{diag}(p_{i,j})}\Delta \Theta_{\cdot,k'}\right) \left(\x_{i,j} \x_{i,j}^\top\right)_{k,k'} \\
    & \qquad - \langle \Delta\Theta_{\cdot,k}^{\top} p_{i,j}, \Delta\Theta_{\cdot,k'}^{\top} p_{i,j}\rangle (\x_{i,j} \x_{i,j}^{\top})_{k,k'}\\
    & = \frac{1}{n} \sum_{i \in \cS} \sum_{j=1}^{M_i} \Bigg( \text{Tr}\Big(\sqrt{\text{diag}(p_{i,j})}\Delta \Theta^\top (\x_{i,j} \x_{i,j}^\top) \sqrt{\text{diag}(p_{i,j})}\Delta \Theta \Big) \\
    & \qquad\qquad - \text{Tr}\Big(p_{i,j}^\top\; \Delta\Theta\; \x_{i,j}\x_{i,j}^{\top}\;\Delta\Theta\; p_{i,j}\Big) \Bigg)\\
    & = \frac{1}{n} \sum_{i \in \cS} \sum_{j=1}^{M_i} \text{Tr} \bigg( \x_{i,j}^\top \Delta\Theta \Big(\text{diag}(p_{i,j}) - p_{i,j}p_{i,j}^\top \Big) \Delta\Theta^\top \x_{i,j}  \bigg)\\
    & \geq \frac{1}{n} \sum_{i \in \cS} \sum_{j=1}^{M_i} \text{Tr} \bigg( \x_{i,j}^\top \Delta\Theta \left( \frac{e^{-2\alpha}}{L} \bI_{L\times L} - \frac{e^{-2\alpha}}{L^2} \1\1^{\top} \right) \Delta\Theta^\top \x_{i,j}  \bigg)
\end{align*}
Now observe that $\Delta \Theta \1 = 0$ follows from Assumption~\ref{asmp:feature} and solution $\hat{\Theta}$. Therefore,
\begin{align*}
    \langle \hat{\Theta} - \Theta^*, \nabla^2 \cL_{\cS}(\Theta), \hat{\Theta} - \Theta^* \rangle & \geq \frac{e^{-2\alpha}}{nL} \sum_{i \in \cS} \sum_{j=1}^{M_i} \text{Tr} (\Delta\Theta^\top \x_{i,j}\x_{i,j}^\top \Delta\Theta )\\
    & = \frac{e^{-2\alpha}}{L} \text{Tr}(\Theta^\top \Sigma_{\cS} \Theta)\\
    & = \frac{e^{-2\alpha}}{L} \text{Tr}\Big(\Theta^\top \sqrt{\Sigma_{\cS}} \sqrt{\Sigma_{\cS}} \Theta\Big)\\
    & = \frac{e^{-2\alpha}}{L} \|\Sigma \Delta \Theta\|_{F}^2\\
    & \geq \frac{e^{-2\alpha}}{L^2}\bigg(\sum_{\ell} \|\Delta \theta_{\ell}\|_{\Sigma_{\cS}}\bigg)^2
\end{align*}
\end{proof}

\gradientBound*
\begin{proof}
First observe that $\|\nabla_{\ell} \cL_{\cS}(\Theta)\|_{\bar{\Sigma}_{\cS_n}^{-1}}^2 = n \|\nabla_{\ell} \cL_{\cS}(\Theta)\|_{{\Sigma}_{\cS_n}^{-1}}^2$ where ${\Sigma}_{\cS_n} = \frac{1}{n} \bar{\Sigma}_{\cS_n}$. Next recall that the gradient is given by
\begin{align*}
        \nabla \cL_{\cS}(\Theta)
        &= \frac{1}{n}\sum_{i \in \cS} \sum_{j \in [M_i]} \vec\Big(\x_{i,j} \otimes \big(p(y_{i,j}| \x_{i,j};\Theta) - \1(y_{i,j})\big) \Big).
\end{align*}
Therefore
\begin{align*}
    \nabla_{\ell} \cL_{\cS}(\Theta)
        &= \frac{1}{n}\sum_{i \in \cS} \sum_{j \in [M_i]} \x_{i,j}  \big(p(y_{i,j} = \ell| \x_{i,j};\Theta) - \I(y_{i,j} = \ell)\big).
\end{align*}
Define $X \in \R^{n M_i \times d}$ as the matrix whose rows are $\x_{i,j}, i \in \cS, j \in [M_i]$, and $V^{\ell}$ be the $n M_i$ dimensional vector whose entries are $p(y_{i,j} = \ell| \x_{i,j};\Theta) - \I(y_{i,j} = \ell)$, i.e.,
\begin{align*}
    V^{\ell}_{ij} = \frac{\exp(\Theta_{\ell}^\top \x_{i,j})}{\sum_{\ell' = 1}^{L}\exp(\Theta_{\ell'}^\top \x_{i,j})} - \I(y_{i,j} = \ell).
\end{align*}
Note that $\E [V^{\ell}] = 0$ and $\big|V^{\ell}_{ij}\big| \leq 2$, which implies $V$ is 4 sub-Gaussian. Therefore
\begin{align*}
    \big\|\nabla_{\ell} \cL_{\cS}(\Theta) \big\|^2_{\Sigma_{\cS}^{-1}} = \frac{1}{n^2} (V^{\ell})^{\top} X \Sigma_{\cS} X^\top V^{\ell} \leq \frac{1}{n} \|V^{\ell}\|^2_2
\end{align*}
Using Bernstein's inequality, with probability $1-\delta,$ for some constant $C>0$
\begin{align*}
    \big\|\nabla_{\ell} \cL_{\cS}(\Theta) \big\|^2_{\Sigma_{\cS}^{-1}} \leq C \; \frac{(d + \log(1/\delta))}{n}
\end{align*}
Taking a union bound over all $\ell \in [L]$ we have with probability $1-\delta$, for some constant $C>0$
\begin{align*}
    \sup_{\ell \in [L]} \big\|\nabla_{\ell} \cL_{\cS}(\Theta) \big\|^2_{\Sigma_{\cS}^{-1}} \leq C \; \frac{(d + \log(L/\delta))}{n}
\end{align*}
which implies we have with probability $1-\delta$, for some constant $C>0$
\begin{align*}
    \sup_{\ell \in [L]} \big\|\nabla_{\ell} \cL_{\cS}(\Theta) \big\|_{\bar{\Sigma}_{\cS}^{-1}} \leq C \; \sqrt{(d + \log(L/\delta))}
\end{align*}
\end{proof}

\end{document}